\documentclass[twoside]{article}

\usepackage[preprint]{aistats2027}
\usepackage[round]{natbib}

\usepackage{amsmath, amssymb, amsthm}
\usepackage{bm, dsfont}
\usepackage{mathtools}
\usepackage{hyperref}
\usepackage{xcolor}
\usepackage{tcolorbox}
\usepackage{caption}
\usepackage[labelformat=simple]{subcaption}
\usepackage{siunitx}
\usepackage[group-minimum-digits = 4, group-separator = {,}]{siunitx}
\usepackage[ruled]{algorithm2e}
\usepackage{xspace}
\usepackage{booktabs, threeparttable}

\mathtoolsset{showonlyrefs, showmanualtags}
\DeclareCaptionType{proc}[Procedure]

\usepackage[columnwise, switch]{lineno}
\theoremstyle{plain}
\newtheorem{theorem}{Theorem}[section]

\newtheorem{lemma}[theorem]{Lemma}

\theoremstyle{definition}

\theoremstyle{remark}

\newcommand*{\tr}{\mathrm{tr}}
\newcommand*{\const}{\mathrm{const}.}
\newcommand{\argmin}{\mathop{\rm arg~min}\limits}

\newcommand*{\cA}{\mathcal{A}}
\newcommand*{\cJ}{\mathcal{J}}
\newcommand*{\cY}{\mathcal{Y}}

\newcommand{\ttM}{\texttt{MovieLens}\xspace}
\newcommand{\ttI}{\texttt{Instacart}\xspace}

\begin{document}


\twocolumn[

\aistatstitle{Scalable Learning of Determinantal Point Processes\\ with Kernel-Structured Interactions}

\aistatsauthor{ Takahiro Kawashima \And Hideitsu Hino }

\aistatsaddress{ ZOZO Research \And Institute of Statistical Mathematics \\ Waseda University} ]

\begin{abstract}
    Determinantal point processes (DPPs) on a finite set are probability models that exhibit repulsion where inter-item similarities are encoded by a positive definite kernel matrix.
    Scaling DPP learning to large datasets typically requires imposing structural constraints on the kernel matrix, such as a low-rank-plus-diagonal form.
    However, structural assumptions introduce a severe trade-off between expressive power and computational scalability. Our study addresses this challenge by (a) modeling the full-rank interactions via a nonlinear kernel function and latent features, (b) designing a surrogate function for the maximum likelihood estimator, and (c) adopting stochastic optimization using random features of the nonlinear kernel.
    Numerical experiments on real-world datasets demonstrate that our method significantly improves this trade-off, achieving better generalization performance with remarkable parameter efficiency.
\end{abstract}

\section{Introduction}
Diversity has emerged as a fundamental concept across various domains of machine learning.
For instance, on Spotify, the diversity of a user's consumption is positively associated with downstream KPIs such as conversion and retention \citep{anderson2020}.
Furthermore, diversifying pre-training data has likewise been shown to improve downstream performance in LLMs \citep{tirumala2023, chen2024}.

Determinantal Point Processes (DPPs) are elegant probabilistic models that naturally capture such diversity.
Originally introduced in statistical mechanics \citep{macchi1975}, their mathematical foundations as general statistical models were established by \citet{soshnikov2000} and \citet{shirai2000}.
Following the comprehensive review by \citet{kulesza2012} that bridged DPPs with the machine learning community, their practical applications have expanded remarkably.
Today, DPPs are widely utilized in recommender, search, retrieval systems \citep{wilhelm2018, warlop2019, kulesza2011, sogi2025}, stochastic optimization \citep{zhang2017a, zhang2019a}, and black-box optimization \citep{nava2022}.
This surge in practical applications has, in turn, renewed theoretical interest, motivating active research from novel perspectives such as geometric analysis \citep{brunel2017, hino2024a, hino2026, anzai2026}.

Parameter learning is a crucial challenge for expanding the applicability of DPPs.
Since DPPs are defined through a determinant and a positive (semi-)definite kernel matrix, learning their parameters without any structural assumptions requires cubic computational complexity with respect to the problem size.
This complexity hinders the applicability of DPPs in the big data era.
To address this scaling issue, previous works have introduced specific structures for the kernel matrix, such as Kronecker-decomposable \citep{mariet2016} and low-rank models \citep{gartrell2017}.
Low-rank models are often augmented with a positive diagonal term for numerical stability and support for large subsets.
Although the resulting low-rank-plus-diagonal (LR+D) kernel is full rank, its off-diagonal interactions are still governed by the rank-$r$ component.
These methods, however, suffer from a trade-off between computational cost and expressive power.

To improve this trade-off, we introduce Kernel-structured DPPs (KS-DPPs) and develop their scalable learning algorithm.
In KS-DPPs, we parameterize the kernel matrix using $(d-1)$-dimensional latent features and a nonlinear kernel function, resulting in $d$ learnable parameters per item.
Unlike LR+D methods, which restrict the off-diagonal interaction component to rank $r$,
our nonlinear parameterization maintains a full-rank interaction structure.
With an appropriate choice of the kernel function, this approach yields high expressive power with a markedly small number of parameters.
Moreover, we design a local surrogate objective inspired by the minorization-maximization (MM) algorithm.
This surrogate admits stochastic optimization via the random features of the kernel, yielding a scalable learning algorithm for KS-DPPs that has linear per-iteration time complexity.

Our main contributions are summarized as follows:
\begin{itemize}
    \item \textbf{Kernel-Structured DPPs:} We formulate large-scale DPP learning through item-wise latent features and qualities under a nonlinear kernel. This formulation can represent full-rank interactions without an explicit rank-$r$ factorization, using only $Nd$ learnable parameters.
    \item \textbf{Scalable Learning and Approximation Analysis:} We develop a scalable learning algorithm with linear per-iteration time complexity by designing a local quadratic surrogate objective and leveraging a quality-weighted Nystr\"om approximation. Furthermore, we establish a theoretical upper bound on the approximation bias.
    \item \textbf{Empirical Validation of Trade-off Improvement:} Empirical evaluations on real-world datasets demonstrate that KS-DPPs achieve superior generalization performance with remarkable parameter efficiency, significantly improving the trade-off between expressive power and computational cost.
\end{itemize}

\subsection*{Related Work}
The learning problem of DPPs was first considered by \citet{kulesza2011a}, although the learnable parameters were limited to the item qualities.
\citet{affandi2014} considered the same nonlinear kernel form as KS-DPPs and developed Bayesian inference for its kernel parameters.
However, their formulation treats item features as observed inputs and infers shared kernel hyperparameters, rather than learning item-specific latent variables.
In contrast, several full-kernel learning algorithms have been developed \citep{gillenwater2014, mariet2015, kawashima2023, castella2026},
and the statistical properties of the maximum likelihood estimator have also been studied \citep{brunel2017, grigorescu2022, castella2026, anzai2026}.
However, these full-kernel learning algorithms require cubic computational complexity with respect to the problem size and are not scalable to large datasets.

Due to this difficulty, some studies introduced special structures for the kernel matrix.
\citet{gartrell2017} and \citet{osogami2018} developed low-rank factorizations of the kernel matrix, and \citet{mariet2016} proposed a Kronecker-decomposable structure for ground sets formed by a direct product.
\citet{dupuy2018} developed a sublinear time method for learning DPPs but it requires strong assumptions: the feature map $\phi$ and the base distribution $p$ of the feature $x$ must be explicitly known, and the second moment $\sum_x p(x) \phi(x) \phi(x)^\top$ needs to be obtained in closed form.
Related to the low-rank model, \citet{gartrell2019a} proposed Deep DPPs, which use a neural network to map item identities and optional metadata to low-rank embeddings and train this mapping end-to-end by maximizing the DPP likelihood.
Despite the nonlinear mapping, the resulting kernel still retains the low-rank form.

While KS-DPPs share a similar motivation with low-rank factorizations and sublinear time methods, our modeling is based on more flexible structural assumptions.
Because of the nonlinearity of the kernel, KS-DPPs can represent full-rank interactions among items even with low-dimensional latent features.
Furthermore, the kernel function can be chosen to control the inductive bias, such as smoothness, stationarity, and periodicity.

\section{Kernel-Structured DPP}\label{sec:methodology}
\subsection{Determinantal Point Processes}\label{sec:dpp}
Determinantal point processes (DPPs) are probability models over point configurations that exhibit repulsion, originally proposed by \citet{macchi1975} in the context of statistical mechanics.
While DPPs are generally defined on uncountable spaces, their probability functions are explicitly formulated for the finite ground set $\cY \coloneq \{1, \ldots, N\}$.
A DPP on $\cY$ assigns a probability to every subset of $\cY$ by
\begin{align}
    \label{eq:l_ensemble}
    P(\cA; \bm{L}) = \frac{\det \bm{L}[\cA]}{\det(\bm{I} + \bm{L})},
\end{align}
where $\bm{L} \in \mathbb{S}^N_+$ is a kernel parameter and $\bm{L}[\cA]$ denotes the principal submatrix of $\bm{L}$ indexed by $\cA \subseteq \cY$ \citep{borodin2005}.
%If $\mathrm{rank}(\bm{L}) < N$, $P(\cA; \bm{L}) = 0$ for $|\cA| > \mathrm{rank}(\bm{L})$.
When applying the LR+D structure as $\bm{L}=\lambda \bm{I}+\bm{V}\bm{V}^\top$ with $\lambda>0$ and $\bm{V}\in\mathbb R^{N\times r}~ (r<N)$, the overall kernel is full-rank, but the inter-item interactions are governed by the rank-$r$ off-diagonal component.

In this paper, we address the learning problem for the parameters $\bm{\theta}$ of the kernel matrix $\bm{L}_{\bm{\theta}}$ given the observations $\cA_1, \ldots, \cA_M \subseteq \cY$.
Our aim is to maximize the regularized log-likelihood:
\begin{align}
    \label{eq:loss}
    \ell_\lambda (\bm{\theta}) \coloneq \frac{1}{M} \sum^M_{m=1} \log P(\cA_m; \bm{L}_{\bm{\theta}}) - \lambda R (\bm{\theta}),
\end{align}
where $R$ is a regularizer and $\lambda > 0$ is a regularization parameter.

Without imposing structural assumptions on the kernel matrix, evaluating exact log-determinants and their gradients inherently requires $O(N^3)$ operations,
prohibiting the application of full-rank DPPs to large-scale problems.
While existing structural constraints like LR+D alleviate this scaling bottleneck, they lead to a severe trade-off between generalization or expressive power and computational cost.
Our goal is to develop a highly scalable learning framework for DPPs without significant loss of expressive power.

\subsection{Parameterization of KS-DPP}
We propose Kernel-structured DPPs (KS-DPPs) that balance expressive power and learning scalability.
Let $\mathcal{X}$ be an input space and $k: \mathcal{X}\times\mathcal{X} \to \mathbb{R}$ be a positive definite kernel admitting a random feature representation \citep{rahimi2008} with a feature map $\bm{\psi}_{\bm{\omega}}$:
\begin{align}
    \label{eq:feature_map}
    k(\bm{x}, \bm{y}) = \mathbb{E}_{\bm{\omega}}[\bm{\psi}_{\bm{\omega}}(\bm{x})^\top \bm{\psi}_{\bm{\omega}}(\bm{y}) ].
\end{align}
We also assume some regularity conditions:
$\nabla_{\bm{x}} k(\bm{x}, \bm{y}) = \mathbb{E}[\nabla_{\bm{x}}\bm{\psi}_{\bm{\omega}}(\bm{x})^\top \bm{\psi}_{\bm{\omega}}(\bm{y})]$ and
$\mathbb{E}\lVert \bm{\psi}_{\bm{\omega}}(\bm{x}) \rVert^2, \mathbb{E}\lVert \nabla_{\bm{x}} \bm{\psi}_{\bm{\omega}} (\bm{x}) \rVert^2 < \infty$.

Prime examples of $k$ are shift-invariant kernels, whose feature maps can be constructed from random Fourier features (RFFs) \citep{rahimi2007}.
In RFFs, the feature maps are
$\bm{\psi}_{\bm{\omega}}(\bm{x}) = (\cos(\bm{\omega}^\top \bm{x}), \sin(\bm{\omega}^\top \bm{x}))^\top$
where $\bm{\omega} \sim p(\bm{\omega})$ is the random vector following the probability density $p$ that determined by the Fourier transform of the shift-invariant kernel $k(\bm{x}, \bm{y}) = k(\bm{x} - \bm{y}) = k(\bm{r})$.
Other valid choices for the kernel include those with random binning features \citep{rahimi2007, wu2016}, random Maclaurin features \citep{kar2012}, or random kernel features \citep{atarashi2019}.

Given a kernel $k$ that satisfies the above conditions, we define the kernel elements as
\begin{align}
    \label{eq:l_structure}
    L_{\bm{\theta}, ij} \coloneq \varepsilon \delta_{ij} + q_i q_j k(\bm{x}_i, \bm{x}_j),
\end{align}
where $\bm{x}_1, \ldots, \bm{x}_N \in \mathcal{X} \subseteq \mathbb{R}^{d-1}$ are latent item features, $q_1, \ldots, q_N > 0$ are item qualities, and $\varepsilon > 0$ is a small hyperparameter to prevent rank degeneracy.\footnote{If $k$ is shift-noninvariant, it should be normalized as $k(\bm{x}, \bm{y}) / \sqrt{k(\bm{x}, \bm{x}) k(\bm{y}, \bm{y})}$ to avoid indeterminacy between the item qualities and the diagonal values of $k$. We assume $k$ is shift-invariant for simplicity.}
This formulation involves $Nd$ learnable parameters $\bm{\theta} = \{(q_i, \bm{x}_i)\}^N_{i=1}$ that are typically unknown.
In contrast to the settings studied by \citet{affandi2014}, where item features are observed and kernel hyperparameters are inferred, we treat $\bm{x}_1,\ldots,\bm{x}_N$ as item-specific latent variables and learn them jointly with the qualities $q_1,\ldots,q_N$.
In matrix notation, \eqref{eq:l_structure} is expressed as
\begin{align}
    \label{eq:l_structure_mat}
    \bm{L}_{\bm{\theta}} = \varepsilon \bm{I} + \bm{Q} \bm{\Sigma} \bm{Q},
\end{align}
where $\bm{Q} \coloneq \mathrm{diag}(q_1, \ldots, q_N)$ and $\bm{\Sigma} \coloneq (k(\bm{x}_i, \bm{x}_j))^N_{i,j=1}$.
Unlike the ordinary LR+D structure, our parameterization in \eqref{eq:l_structure} induces full-rank interactions derived from the $(d-1)$-dimensional kernel geometry.
Because low-dimensional latent features often underlie complex data, this structure is expected to achieve high parameter efficiency in many applications \citep{bengio2013, zhang2019b}.

\section{Learning KS-DPPs}\label{sec:learning}
\subsection{Surrogate Objective}\label{sec:surrogate}
Our strategy for learning KS-DPPs is to construct a surrogate function of $\ell_\lambda (\bm{\theta})$, whose stochastic gradients can be efficiently obtained using the random features of $k$, and to maximize this surrogate using stochastic optimization.
To this end, we leverage the minorization-maximization (MM) algorithm \citep{hunter2004, sun2017}.
In the MM framework, an objective $f: \theta \mapsto f(\theta)$ is maximized by iteratively optimizing a minorizer $g(\cdot; \theta_t)$ that satisfies $g(\theta; \theta') \leq f(\theta)$ and $g(\theta' ; \theta') = f(\theta')$ for all $\theta, \theta'$.
By updating $\theta_{t+1} \gets \mathrm{argmax}_{\theta} ~g(\theta ; \theta_t)$, monotonic improvement is guaranteed as $f(\theta_{t+1}) \geq g(\theta_{t+1} ; \theta_{t}) \geq g(\theta_t ; \theta_t) = f(\theta_t)$.
This approach simplifies the original maximization problem of $f$ into a sequence of surrogate optimizations.

Following the MM principle, we construct a minorizer for the DPP log-likelihood:
\begin{align}
    \label{eq:loglik}
    \log P(\cA; \bm{L}) = \log \det \bm{L}[\cA] - \log \det(\bm{I} + \bm{L}),
\end{align}
where $\bm{L}=\bm{L}_{\bm{\theta}}$.
Since $\log\det$ is concave on $\mathbb{S}_{++}$, the second term is minorized by the tangent plane at the current iterate $\bm{L}_t \coloneq \bm{L}_{\bm{\theta}_t}$:
\begin{align}
    &-\log\det(\bm{I} + \bm{L})\\
    &\quad\geq -\log\det(\bm{I}+\bm{L}_t) - \tr \{(\bm{I} + \bm{L}_t)^{-1}(\bm{L} - \bm{L}_t)\}\\
    \label{eq:tangent}
    &\quad=- \tr \{(\bm{I} + \bm{L}_t)^{-1}\bm{L} \} + \mathrm{const.},
\end{align}
as utilized in \citet{kawashima2023}.

Conversely, the first term $\log\det\bm{L}[\cA]$ is more challenging to minorize because of its concavity.
One option is to use a global curvature bound; since $\bm{L} \succeq \varepsilon \bm{I} \succ 0$ by the structural assumption \eqref{eq:l_structure_mat},
the gradient $\nabla_{\bm{L}}\log\det\bm{L}[\cA]$ is Lipschitz continuous, which induces a quadratic minorizer on the domain $\bm{L} \succeq \varepsilon\bm{I}$.
This follows from the descent lemma \citep{bertsekas1999}, one of the standard techniques for constructing MM surrogates \citep{sun2017}.
However, because $\varepsilon > 0$ is typically very small, this approach yields an excessively loose and conservative minorizer, resulting in prohibitively slow convergence.

To address this, we consider a local quadratic minorizer for $\log\det\bm{L}[\cA]$.
We define the normalized relative change of $\bm{L}$ from $\bm{L}_{t}$:
\begin{align}
    \bm{E}_{t} \coloneq \bm{L}_t[\cA]^{-\frac{1}{2}} (\bm{L}[\cA] - \bm{L}_t[\cA]) \bm{L}_t[\cA]^{-\frac{1}{2}},
\end{align}
leading to $\log\det\bm{L}[\cA] = \log\det(\bm{I} + \bm{E}_t) + \const$.
We impose a trust region $\lVert \bm{E}_t \rVert_{\mathrm{op}} \leq \rho$ for a given $\rho \in (0, 1)$, where $\lVert \cdot \rVert_{\mathrm{op}}$ is the spectral norm.
This constraint is equivalent to restricting $\bm{L}[\cA]$ within the range $(1-\rho)\bm{L}_t[\cA] \preceq \bm{L}_{\bm{\theta}}[\cA] \preceq (1+\rho)\bm{L}_t[\cA]$.

Let $\eta_j$ be the $j$-th eigenvalue of $\bm{E}_t$.
If there exists a coefficient $\beta_\rho$ such that
$\log(1 + \eta) \geq \eta - \frac{\beta_\rho}{2}\eta^2$
for $\eta \in [-\rho, \rho]$,
then
$\log\det(\bm{I} + \bm{E}_t) = \sum_j \log (1+ \eta_j)\geq \sum_j(\eta_j - \frac{\beta_\rho}{2}\eta^2_j) =  \tr \bm{E}_t - \frac{\beta_\rho}{2} \lVert \bm{E}_t\rVert^2_F$
holds,
where $\beta_\rho$ serves as a local curvature bound around $\bm{L}_t$.
Among all quadratic minorizers that touch $\log(1+\eta)$ at $\eta=0$, the constant and linear coefficients are necessarily $0$ and $1$, respectively.
The tightest $\beta_\rho$ satisfying the inequality above is given by
$\beta_\rho = \frac{-2(\rho + \log (1 - \rho))}{\rho^2}$.

Overall, we minorize the objective \eqref{eq:loss} within the trust region by
\begin{align}
    g_\lambda(\bm{\theta}; \bm{\theta}_t) &= \frac{1}{M}\sum^M_{m=1} \left ( \tr \bm{E}_{t,m} - \frac{\beta_\rho}{2} \lVert \bm{E}_{t,m}\rVert^2_F\right )\\
    \label{eq:surrogate}
    &\quad- \tr\left\{ \left (\bm{I} + \bm{L}_t\right )^{-1}\bm{L}_{\bm{\theta}} \right\} - \lambda R(\bm{\theta})
\end{align}
up to a constant, where
$
    \bm{E}_{t,m} \coloneq \bm{L}_t[\cA_m]^{-\frac{1}{2}} (\bm{L}[\cA_m] - \bm{L}_t[\cA_m]) \bm{L}_t[\cA_m]^{-\frac{1}{2}}.
$
The entire learning framework involves a nested procedure:
maximizing $g_\lambda (\bm{\theta}; \bm{\theta}_t)$ via the inner steps,
and updating $\bm{\theta}_{t+1} \gets \mathrm{argmax}_{\bm{\theta}} ~g_\lambda (\bm{\theta}; \bm{\theta}_t)$ as the $t$-th outer step.

\subsection{Stochastic Gradient and Kernel Approximation}
Computing the exact gradient of the surrogate \eqref{eq:surrogate} requires $O(N^3 + M\kappa^3)$ operations due to matrix inversions and multiplications, where $\kappa \coloneq \max_m |\cA_m|$.
To make the optimization scalable, we develop a linear-time method for evaluating the stochastic gradients.
The precise formulation of this gradient estimator is detailed in Appendix \ref{app:gradient_estimators}.

Observe that $\bm{E}_{t,m}$ is affine with respect to $\bm{L}_{\bm{\theta}}$, because $\bm{L}_t[\mathcal{A}_m]^{-1/2}$ remains fixed throughout the $t$-th outer step. 
Consequently, the local minorizer \eqref{eq:surrogate} is purely quadratic in $\bm{L}_{\bm{\theta}}$, and crucially, it avoids any inversion of the moving iterate $\bm{L}_{\bm{\theta}}$.
Since $k$ admits the feature representation \eqref{eq:feature_map}, both $\bm{L}_{\bm{\theta}}$ and its gradient $\nabla \bm{L}_{\bm{\theta}}$ can be unbiasedly estimated using random features.
Because the gradient of this quadratic surrogate yields the bilinear forms of $\bm{L}_{\bm{\theta}}$ and $\nabla \bm{L}_{\bm{\theta}}$,
constructing an unbiased estimator requires two independent draws $\bm{\omega},\bm{\nu} \sim p$.
The remaining bottleneck for a linear-time Monte Carlo evaluation of $\nabla g_\lambda$ is solving linear systems including a linear system of size $N$ appearing in
\begin{align}
    \label{eq:grad_denom}
\nabla_{\theta_i}\mathrm{tr}\{(\bm I+\bm L_t)^{-1}\bm L_{\bm\theta}\}=\mathrm{tr}\{(\bm I+\bm L_t)^{-1}\nabla_{\theta_i}\bm L_{\bm\theta}\}
\end{align}

To evaluate \eqref{eq:grad_denom} at scale, we employ the quality-weighted Nystr{\"o}m approximation for $\bm{L}_t$, which provides an approximation $\widehat{\bm{L}}_t \approx \bm{L}_t$ with a configurable rank $r_{\mathrm{Nys}}$ \citep{williams2000, drineas2005}.
Note that Nystr\"om approximations have previously been used to construct low-rank surrogates for scalable sampling from a fixed DPP \citep{affandi2013}.
Here, we use the approximation for a different purpose: to cache the frozen coefficient of the normalizer tangent during the inner steps.

Let $\bm{M}_t \coloneq \bm{Q}_t \bm{\Sigma}_t \bm{Q}_t $ and define the sampling weights $p_i = q^2_{t,i} / \lVert \bm{Q}_t \rVert^2_F$ for $i = 1, \ldots, N$.
Our quality-weighted Nystr{\"o}m method first samples $r_{\mathrm{Nys}}$ indices from $\cY$ without replacement according to $p_i$.
This procedure can be realized by the Gumbel top-$k$ trick, for example \citep{kool2019}.
Let $J = \{j_1, \ldots, j_{r_{\mathrm{Nys}}}\} \subseteq \cY$ denote the set of sampled indices.
The method then constructs the low-rank approximation:
\begin{align}
    \label{eq:qw_nystrom}
    \bm{M}_t \approx \widehat{\bm{M}}_t \coloneq (\bm{M}_t)^\top_{J, :}\bm{M}_t[J]^+ (\bm{M}_t)_{J, :},
\end{align}
where $(\bm{M}_t)_{J, :}$ is the submatrix of $\bm{M}_t$ comprising the rows indexed by $J$, and $\bm{M}_t[J]^+$ is the pseudo-inverse of $\bm{M}_t[J]$.
The desired kernel approximation is obtained by $\widehat{\bm{L}}_t = \varepsilon \bm{I} + \widehat{\bm{M}}_t$.

Once the approximation $\widehat{\bm{L}}_t$ is obtained at the beginning of each outer step, the approximate inverse $\widehat{\bm{W}}_t \coloneq (\bm{I} + \widehat{\bm{L}}_t)^{-1}$ can be cached and reused during the subsequent inner steps.
Although the approximation violates the touching condition $g(\theta' ; \theta') = f(\theta')$ of the standard MM principle, $g$ still serves as a valid non-touching minorizer of $f$ because the inequality \eqref{eq:tangent} holds for any choice of $\bm{L}_t$.
A key advantage of using the Nystr{\"o}m method here is that this non-touching bias remains theoretically controllable.
We analyze this approximation gap in Section \ref{sec:gap}.

\subsection{Hybrid Linear Solves}
Evaluating the observation-dependent terms also poses difficulties.
By letting $\bm{D}^t_{\bm{\theta}} \coloneq \bm{L}_{\bm{\theta}}-\bm{L}_t$, we have
$\tr\bm{E}_{t,m} = \tr(\bm{L}_t[\cA_m]^{-1}\bm{D}^t_{\bm{\theta}} [\cA_m] )$
and
$\|\bm{E}_{t,m}\|_F^2 = \tr(\bm{L}_t[\cA_m]^{-1}\bm{D}^t_{\bm{\theta}} [\cA_m]\bm{L}_t[\cA_m]^{-1}\bm{D}^t_{\bm{\theta}} [\cA_m])$
by the cyclic property of the trace, thereby avoiding matrix square roots.
Since their gradients involve applying $\bm{L}_t[\cA_m]^{-1}$ to $\bm{D}^t_{\bm{\theta}}$ or $\nabla_{\theta_i} \bm{D}^t_{\bm{\theta}}$,
we need to evaluate
\begin{align}
    \label{eq:unified_linear_system}
    \!\!\!\!
    \bm{\xi}^t_{\boldsymbol{\theta}} \coloneq \bm{L}_t[\cA]^{-1} \bm{\varphi}^t_{\boldsymbol{\theta}}
    = (\varepsilon \bm{I} + \bm{Q}_t [\cA] \bm{\Sigma}_t [\cA] \bm{Q}_t[\cA])^{-1}\bm{\varphi}^t_{\boldsymbol{\theta}}
\end{align}
for each observation $\cA$ in the minibatch,
where $\bm{\varphi}^t_{\boldsymbol{\theta}}$ is a column of $\bm{D}^t_{\boldsymbol{\theta}}$ or its gradient.
In practice, we do not have access to the exact forms of $\bm{D}^t_{\bm{\theta}}$ and $\nabla_{\theta_i} \bm{D}^t_{\bm{\theta}}$, but only to their unbiased random feature representations.
Therefore, $\bm{\varphi}^t_{\bm{\theta}}$ is implemented as the random features of the column of $\bm{D}^t_{\bm{\theta}}$ or its gradient (see Appendix \ref{app:gradient_estimators} for details).

While the observed subset sizes $|\cA_m|$ are often small, occasional large subsets can dominate the overall computational cost.
This motivates us to develop a more efficient solver for \eqref{eq:unified_linear_system} than the standard preconditioned conjugate gradient (PCG) method.

The solution of \eqref{eq:unified_linear_system} is also the unique minimizer of the problem:
\begin{align}
    \label{eq:prob_solve}
    &\bm{\xi}^t_{\bm{\theta}} = \argmin_{\bm{\xi} \in \mathbb{R}^{|\cA|}}h^t_{\bm{\theta}}(\bm{\xi}),\\
    &h^t_{\bm{\theta}}(\bm{\xi}) \coloneq \frac{1}{2}\bm{\xi}^\top (\varepsilon \bm{I} + \bm{Q}_t [\cA] \bm{\Sigma}_t [\cA] \bm{Q}_t[\cA]) \bm{\xi} -  \bm{\xi}^\top \bm{\varphi}^t_{\bm{\theta}}.
\end{align}
Recalling that $\bm{\Sigma}_t[\cA]$ has the feature vector representation
$\bm{\Sigma}_t[\cA] = \mathbb{E}_{\bm{\omega}}[\bm{\Psi}^t_{\bm{\omega}, \cA} \bm{\Psi}^{t\top}_{\bm{\omega}, \cA}]$
with $\bm{\Psi}^t_{\bm{\omega}, \cA} \coloneq (\bm{\psi}_{\bm{\omega}}(\bm{x}^t_i))^\top_{i \in \cA}$,
the gradient of $h^t_{\bm{\theta}}$ can be written as
\begin{align}
    \nabla_{\bm{\xi}} h^t_{\bm{\theta}}(\bm{\xi})
    = \varepsilon \bm{\xi} + \bm{Q}_t[\cA] \mathbb{E}_{\bm{\omega}}[\bm{\Psi}^t_{\bm{\omega}, \cA}\bm{\Psi}^{t\top}_{\bm{\omega}, \cA}] \bm{Q}_t[\cA] \bm{\xi} - \bm{\varphi}^t_{\bm{\theta}},
\end{align}
which can be efficiently approximated via Monte Carlo estimation with $\bm{\omega} \sim p$.
By iteratively updating $\bm{\xi}$ using this stochastic gradient,
we can solve \eqref{eq:unified_linear_system} in $O(|\cA|T_{\mathrm{solve}})$ evaluations of $\bm{\psi}_{\bm{\omega}}(\bm{x}_i)$, where $T_{\mathrm{solve}}$ denotes the fixed number of solver iterations.

Although this stochastic solve strategy is highly stable in practice due to the strong convexity of \eqref{eq:prob_solve}, the finite $T_{\mathrm{solve}}$ leads to an approximation bias.
To mitigate this, we switch between two different linear solvers.
We adopt the stochastic solver above if $|\cA|$ is large (e.g., $|\cA| \geq 10^3$),
and otherwise default to PCG, whose complexity is typically $O(|\cA|^2)$.

\subsection{Gap Analysis}\label{sec:gap}
As in \eqref{eq:tangent}, we use the tangent-based lower bound for the observation-independent term $f_{\mathrm{den}}(\bm{L}) \coloneq -\log\det(\bm{I}+\bm{L})$ within the objective function $\ell_\lambda(\bm{L})$.
While the lower bound guarantees the touching condition $\ell_\lambda(\bm{L}_t) = g_{\lambda}(\bm{L}_t; \bm{L}_t)$ exactly and suffices for the monotonicity of the MM, we intentionally introduce the Nystr{\"o}m approximation $\widehat{\bm{L}}_t \approx \bm{L}_t$ into this term to achieve computational scalability.

Let $\hat{g}_{\lambda} (\bm{L}; \bm{L}_t, \widehat{\bm{L}}_t)$ be the approximated surrogate obtained by replacing the exact tangent bound $f_{\mathrm{den}}(\bm{L}_t)-\tr\{(\bm{I} + \bm{L}_t)^{-1} (\bm{L} - \bm{L}_t)\}$ with its approximation $f_\mathrm{den}(\widehat{\bm{L}}_t)-\tr\{(\bm{I} + \widehat{\bm{L}}_t)^{-1} (\bm{L} - \widehat{\bm{L}}_t)\}$.\footnote{We analyze the gap including the constant terms independent of $\bm{L}$, while the exact surrogate $g_\lambda$ in \eqref{eq:surrogate} is defined up to a constant.}
This substitution yields
$\ell_\lambda (\bm{L}_t) = g_{\lambda}(\bm{L}_t; \bm{L}_t) \geq \hat{g}_{\lambda}(\bm{L}_t; \bm{L}_t, \widehat{\bm{L}}_t)$.
This approximation gap is quantified by the Bregman divergence $D_{f_{\mathrm{den}}}$:
\begin{align}
    \gamma_t \coloneq \ell_{\lambda}(\bm{L}_t) - \hat{g}_{\lambda}(\bm{L}_t;  \bm{L}_t, \widehat{\bm{L}}_t)
    = D_{f_{\mathrm{den}}}(\bm{L}_t \Vert \widehat{\bm{L}}_t) \geq 0.
\end{align}
We refer to $\gamma_t$ as the anchor gap induced by the Nystr{\"o}m approximation.

Let $G_t \coloneq \hat{g}_{\lambda}(\bm{L}_{t+1}; \bm{L}_t, \widehat{\bm{L}}_t) - \hat{g}_{\lambda}(\bm{L}_t; \bm{L}_t, \widehat{\bm{L}}_t)$ denote the surrogate improvement achieved during the $t$-th inner step.
Due to the anchor gap, achieving $G_t \geq \gamma_t$ is sufficient to guarantee the monotonic ascent of the outer iterations.
Fortunately, the anchor gap is typically small and theoretically controllable.
\begin{lemma}\label{lem:gap_bound}
    The anchor gap $\gamma_t$ is upper-bounded by:
    \begin{align}
        \label{eq:gap_bound}
        \gamma_t \leq \frac{1}{2} \lVert (\bm{I} + \widehat{\bm{L}}_t)^{-1/2} (\bm{L}_t - \widehat{\bm{L}}_t) (\bm{I} + \widehat{\bm{L}}_t)^{-1/2}\rVert^2_F.~
    \end{align}
\end{lemma}
\begin{theorem}\label{thm:ex_gap_bound}
    Let $k_0 \coloneq \Sigma_{t, ii}$. The expected anchor gap given the current parameters $\bm{\theta}_t$ admits the following $r_{\mathrm{Nys}}$-dependent upper bound:
    \begin{align}
        \label{eq:ex_gap_bound}
        \mathbb{E}[\gamma_t | \bm{\theta}_t] \leq \frac{k^2_0 \lVert \bm{Q}_t \rVert^4_F}{2(1+\varepsilon)^2 r_{\mathrm{Nys}}},
    \end{align}
    which implies $\mathbb{E}[\gamma_t | \bm{\theta}_t] = O(r^{-1}_{\mathrm{Nys}})$.
\end{theorem}
The proofs are shown in Appendix \ref{app:proof}.

These bounds demonstrate that the anchor gap is bounded by the square of the relative difference between $\bm{L}_t$ and $\widehat{\bm{L}}_t$, ensuring that the gap is favorably small, and decays at a rate of $O(r^{-1}_{\mathrm{Nys}})$.
In Appendix \ref{app:exp}, we evaluate this gap on synthetic data mimicking real-world popularity skew, observing that the theoretical rate is somewhat conservative in this setting.

\subsection{Computational Complexity}\label{sec:complexity}
We analyze the per-outer-step computational complexity of our learning algorithm using RFFs.
We treat the number of Monte Carlo samples and the Nystr{\"o}m rank as constants and omit them from the following analysis.
For simplicity, we assume the use of the stochastic solver for the linear system \eqref{eq:unified_linear_system}.

First, at the beginning of each outer step, we cache the compact Woodbury representation of $\widehat{\bm{W}}_t$ via the quality-weighted Nystr{\"o}m method to approximate \eqref{eq:grad_denom}; this takes $O(Nd)$.
The $O(N \log r_{\mathrm{Nys}})$ cost of Gumbel-top-$k$ sampling is absorbed into this $O(Nd)$ term.
Since the random feature-based unbiased approximation of $\nabla L_{\bm{\theta}}$ requires $O(Nd)$ and applying the cached $\widehat{\bm{W}}_t$ to compute $\tr(\widehat{\bm{W}}_t \nabla\bm{L}_{\bm{\theta}})$ takes $O(N)$,  
the overall approximation of \eqref{eq:grad_denom} scales as $O(Nd)$.

For the gradient of the first term in \eqref{eq:surrogate}, we evaluate the random features and their derivatives for all $\cA_m$ in a minibatch $\mathcal{B}_t$ of size $B$.
Because the dense matrix $\bm{L}[\cA_m]$ is never explicitly formed, this evaluation takes $O(d \sum_{m \in \mathcal{B}_t} |\cA_m|) = O(B\kappa d)$ time.
Solving the linear system \eqref{eq:unified_linear_system} for $m \in \mathcal{B}_t$ imposes an additional $O(B\kappa d  T_{\mathrm{solve}})$ complexity, where the $O(d)$ factor arises from evaluating $\bm{\omega}^\top \bm{x}_i$ for the RFFs.

As a result, our learning algorithm takes
\begin{align}
    &O(\underbrace{Nd}_{\mbox{{\footnotesize caching}}} + T_{\mathrm{in}}(
    \underbrace{Nd}_{\mbox{{\footnotesize obs.-indep.}}}
    + \underbrace{B\kappa d + B\kappa dT_{\mathrm{solve}}}_{\mbox{{\footnotesize obs.-depend.}}})
    )\\
    \label{eq:complexity}
    &\quad= O(T_{\mathrm{in}} (Nd + B\kappa d T_{\mathrm{solve}}))
\end{align}
per outer step, where $T_{\mathrm{in}}$ denotes the fixed number of inner steps.
Treating $T_{\mathrm{in}}$ and $T_{\mathrm{solve}}$ as constants,
the overall complexity reduces to $O(Nd + B \kappa d)$, thereby achieving linear scalability.
Note that $\kappa$ is replaced with $\kappa^2$ if we employ PCG instead of the stochastic solver.
We empirically verify this complexity via runtime measurements in Section \ref{sec:ex}.

As a remark, performing an exact acceptance test for the trust region
$(1-\rho) \bm{L}_t[\cA_m] \preceq \bm{L}_{\bm{\theta}}[\cA_m] \preceq (1 +\rho) \bm{L}_t[\cA_m]$
and introducing a backtracking line search are incompatible with our target complexity.
Since the test for each $m$ constitutes a $|\cA_m| \times |\cA_m|$ spectral problem, it requires $\Omega(M\kappa^3)$ operations.
Fortunately, the step size of the inner stochastic optimization is naturally controlled by the learning rate in our scheme.
For a sufficiently small learning rate, we have $\bm{L}_{\bm{\theta}} \to \bm{L}_t$ and $\lVert \bm{E}_{t,m}\rVert_{\mathrm{op}} \to 0$ for all $m$; thus, $\bm{L}_{\bm{\theta}}$ reliably remains within the desired trust region.
We therefore rely on the learning rate to implicitly enforce the trust-region constraint, omitting the explicit acceptance test in practice.

\section{Numerical Experiments}\label{sec:ex}
\subsection{Dataset and Configuration}
We evaluate our method against the LR+D structure on the \ttM \citep{harper2015} and \ttI datasets \citep{instacart2017}.
We treat movies and grocery products as items ($i \in \cY$), and each user's rating history or prior orders as an observed subset.

For both datasets, we uniformly subsample the observations to form training ($M_{\mathrm{train}} \!=\! \num{10000}$), validation ($M_{\mathrm{val}} \!=\! \num{500}$), and test ($M_{\mathrm{test}}\!=\!\num{1000}$) sets.
For KS-DPP, we use a Gaussian kernel with RFFs, an $L_2$ regularizer, and set hyperparameters $r_{\mathrm{Nys}} \!=\! 32$, $\rho \!=\! 0.2$, $T_{\mathrm{in}} \!=\! 10$, and $T_{\mathrm{solve}} \!=\! 50$.
Both models are trained using Adam \citep{kingma2015} with a minibatch size of $B\!=\!64$.
The learning rate is fixed at $0.01$ for KS-DPP, and tuned to $3 \!\times\! 10^{-4}$ (\ttM) or $5 \!\times\! 10^{-5}$ (\ttI) for LR+D to maximize training stability.

We compare LR+D across ranks $r \in \{4, 8, \ldots, \num{1024}\}$ and KS-DPP across dimensions $d \in \{4, 8, 16, 32\}$.
To aid LR+D convergence, we employ a warm-start strategy, initializing its parameters from the learned results of the preceding smaller $r$.
We also evaluate scalability by sweeping the ground set size $N$ from $\num{3000}$ up to the full catalogue $N_{\max}$ (excluding completely unobserved items), retaining the most frequently observed items first.

Detailed configurations, including more information about the datasets, the number of sampled random features, held-out log-likelihood evaluation procedures, and hardware and software specifications, are provided in Appendix \ref{app:exp_details}.

\subsection{Main Results}

\begin{figure*}[t]
    \centering
    \captionsetup{skip=-5pt}
    \begin{minipage}{0.495\textwidth}
        \centering
        \subcaption{Parameter Efficiency}
        \includegraphics[width = 1.0\columnwidth]{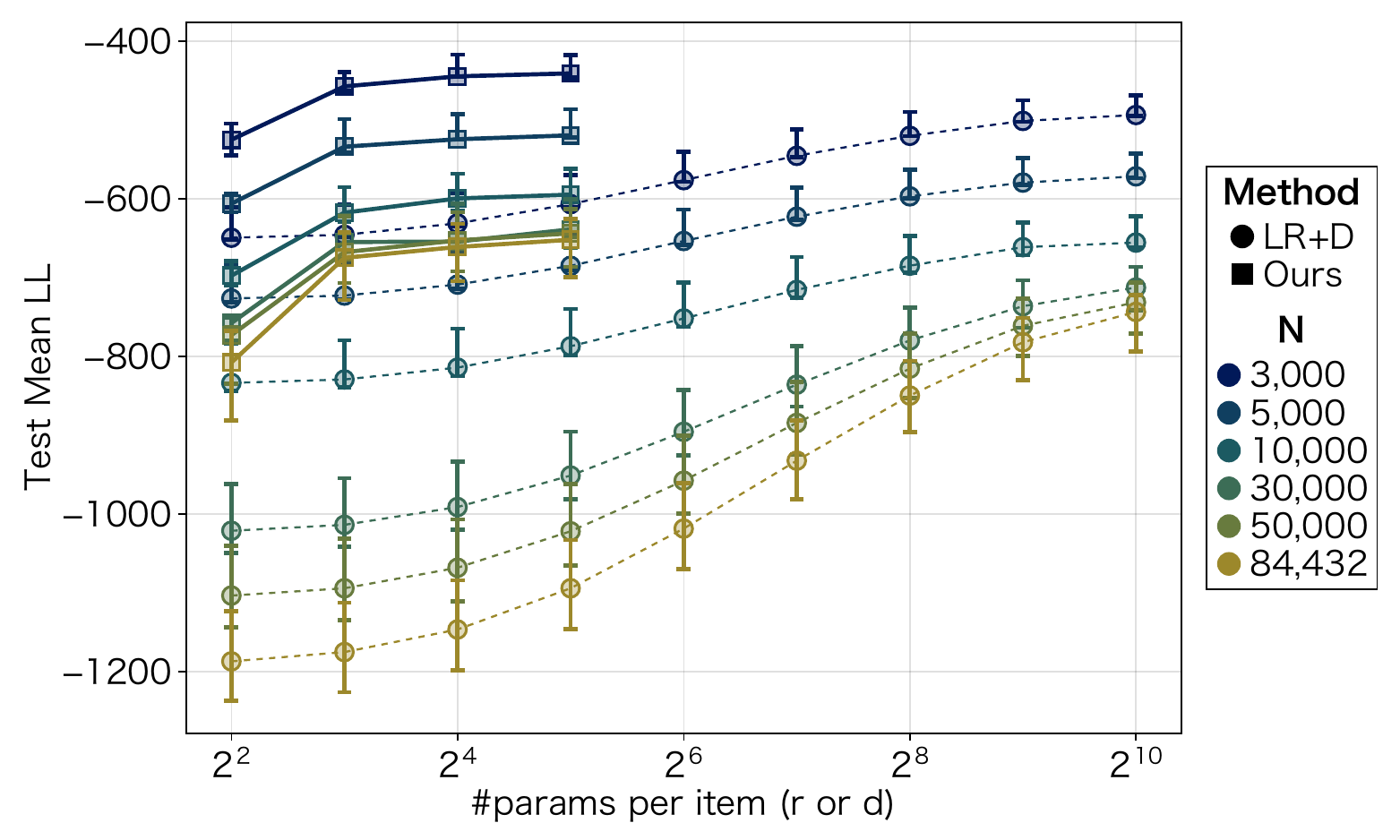}
        \label{subfig:movielens_quality}
    \end{minipage}
    \begin{minipage}{0.495\textwidth}
        \centering
        \subcaption{Convergence Time}
        \includegraphics[width = 1.0\columnwidth]{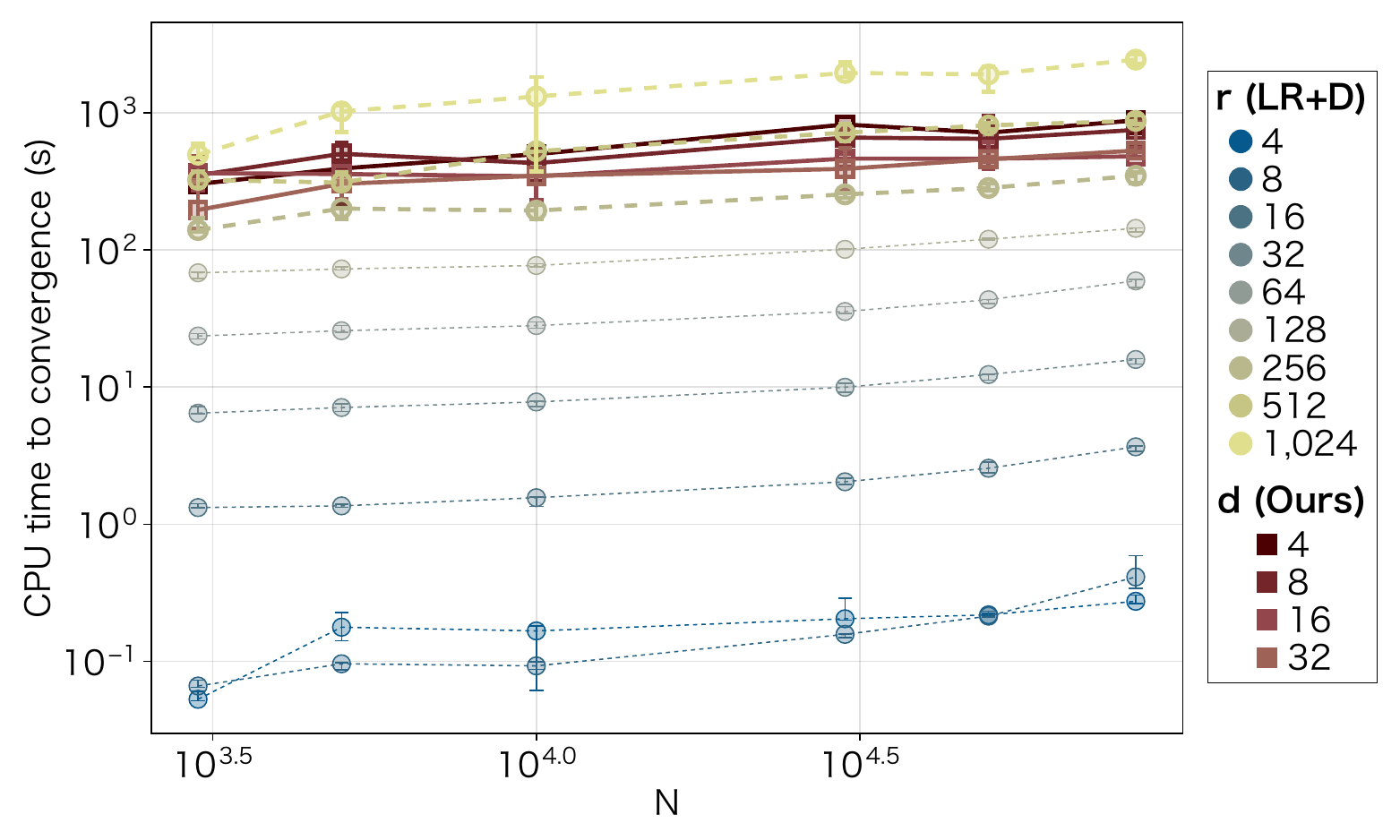}
        \label{subfig:movielens_time}
    \end{minipage}
    \caption{Results on the \ttM dataset. Each marker shows the median and each bar corresponds to minimum/maximum values of $3$ trials.}
    \label{fig:movielens_result}
\end{figure*}

\begin{figure*}[t]
    \centering
    \captionsetup{skip=-5pt}
    \begin{minipage}{0.495\textwidth}
        \centering
        \subcaption{Parameter Efficiency}
        \includegraphics[width = 1.0\columnwidth]{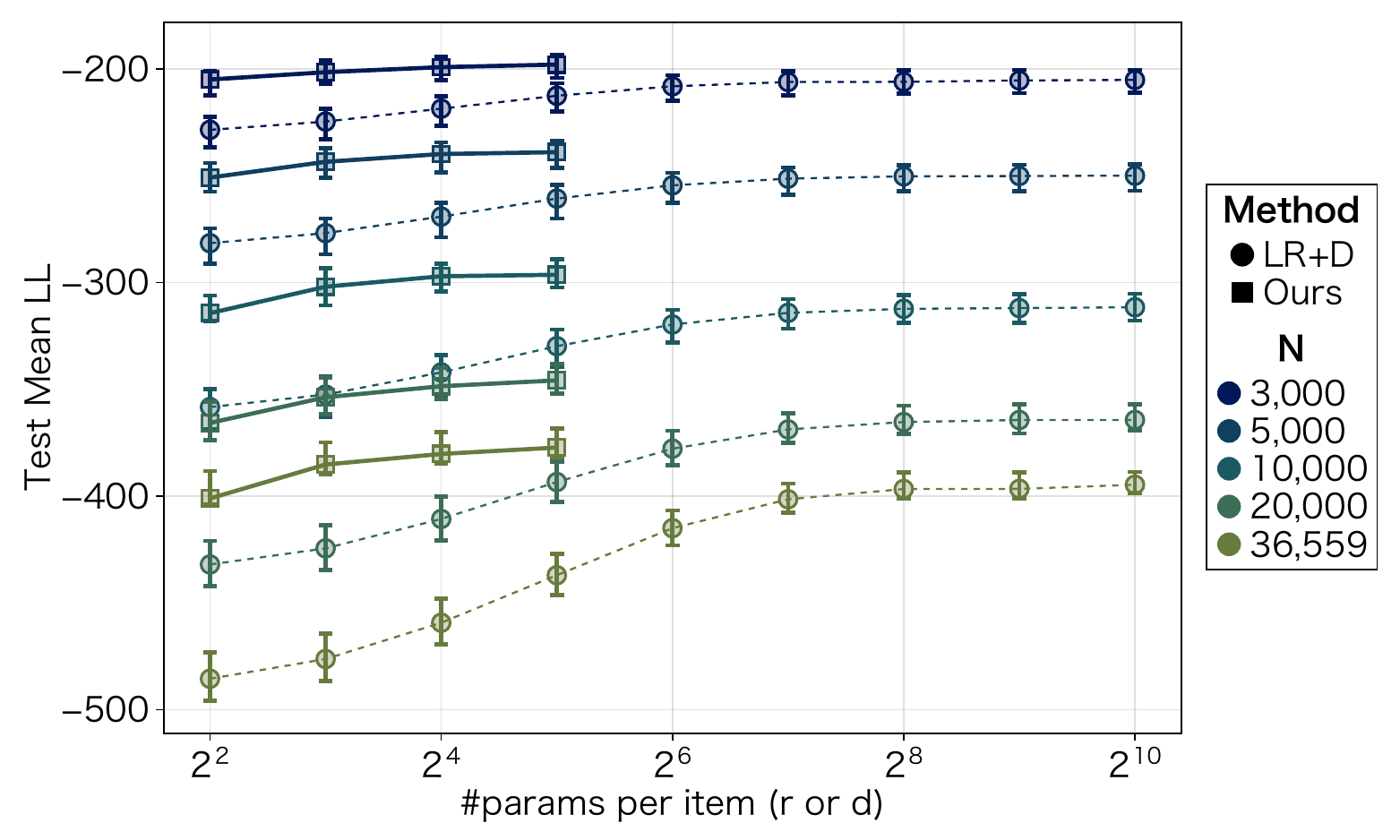}
        \label{subfig:instacart_quality}
    \end{minipage}
    \begin{minipage}{0.495\textwidth}
        \centering
        \subcaption{Convergence Time}
        \includegraphics[width = 1.0\columnwidth]{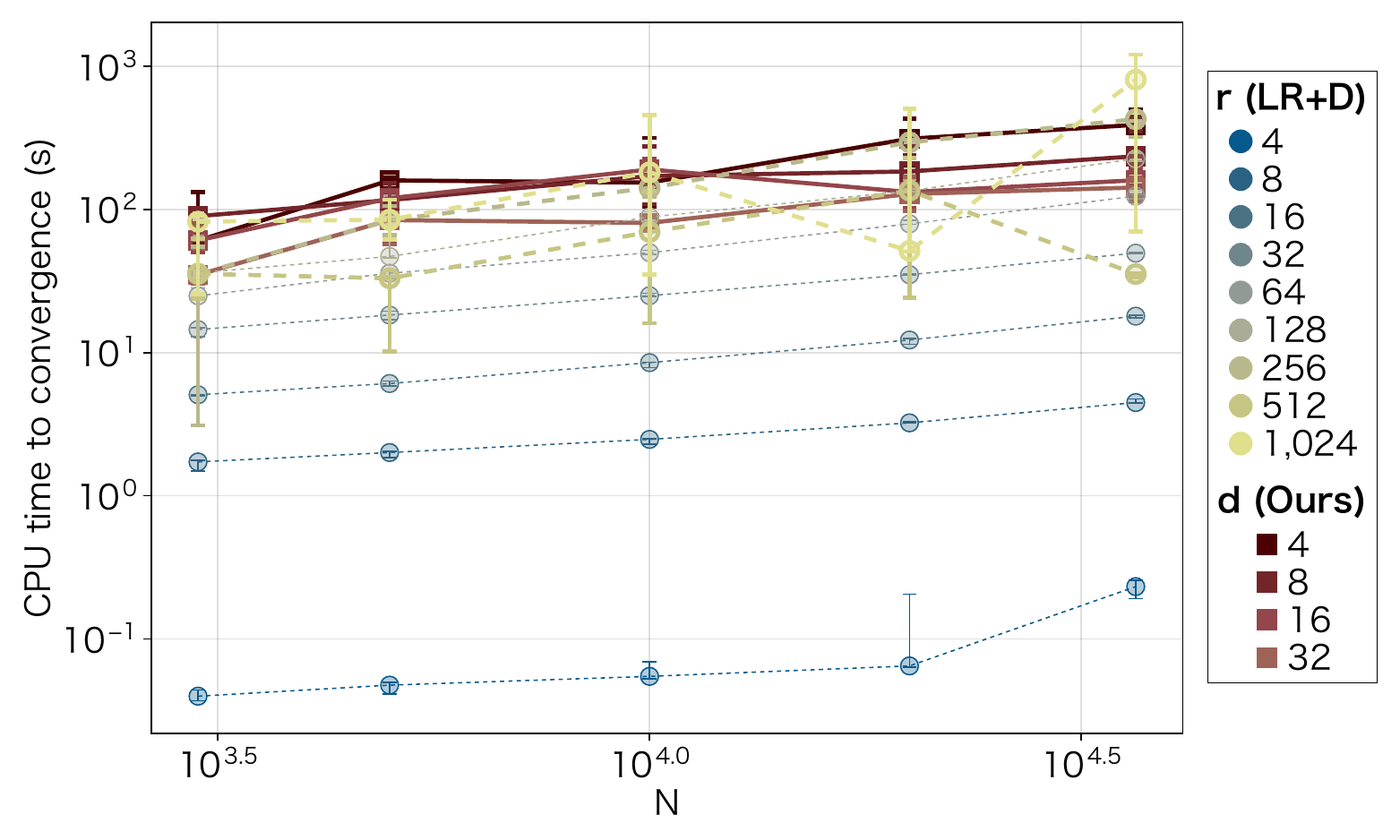}
        \label{subfig:instacart_time}
    \end{minipage}
    \caption{Results on the \ttI dataset. Each marker shows the median and each bar corresponds to minimum/maximum values of $3$ trials.}
    \label{fig:instacart_result}
\end{figure*}

Figures \ref{fig:movielens_result} and \ref{fig:instacart_result} show the results for the \ttM and \ttI datasets, respectively.
As shown in Figures \ref{subfig:movielens_quality} and \ref{subfig:instacart_quality}, KS-DPP with $d=8$ consistently achieves generalization performance comparable to or better than LR+D with $r=\num{1024}$.
This remarkable parameter efficiency arises from the parameterization of the kernel matrix via the nonlinear kernel function.
Conversely, KS-DPP with $d=4$ underperforms compared to $d\geq 8$ on both datasets,
suggesting that their intrinsic dimensionality exceeds $d=4$.

Regarding convergence time, Figure \ref{subfig:movielens_time} shows that KS-DPP is comparable to LR+D with $r \in [256, 512]$ on the \ttM dataset.
Since KS-DPP yields higher test log-likelihoods than LR+D with $r=\num{1024}$, our method significantly improves the trade-off between generalization performance and learning time.
For the \ttI dataset (Figure \ref{subfig:instacart_time}), LR+D exhibits non-monotonic convergence times.
This artifact may stem from the warm-start strategy: because LR+D performance saturates around $r=128$, initializations for larger $r$ already lie on a plateau, causing the optimization to stall.
Nevertheless, KS-DPP with $d=32$ converges faster than LR+D with $r=256$, further supporting the improved computational trade-off.

Interestingly, increasing $d$ tends to yield faster and more stable convergence for KS-DPP, despite the $O(Nd)$ per-outer-step cost scaling (the measured per-outer-step times are reported later).
Although our objective differs from standard matrix factorization, this benign optimization behavior is conceptually linked to the phenomenon where over-parameterization eliminates spurious local minima \citep{ge2017, boumal2020}.
We believe a fundamentally similar mechanism drives the accelerated convergence observed in our experiments.

\subsection{Comparison with Exact Method}
\begingroup
\renewcommand{\arraystretch}{1.15}
\begin{table*}[t]
\setlength\tabcolsep{0pt}
\centering
\caption{Comparison of the test mean log-likelihoods with the exact MM method on \ttM. The values denote $\mbox{med}^{\mbox{max}}_{\mbox{min}}$ of 3 trials.}
\label{tab:comparison_mm}
\begin{threeparttable}
\begin{tabular}{@{\extracolsep{2ex}}*{6}{lccccc}}
\toprule
\multicolumn{1}{c}{\textbf{Method}} & \textbf{N=200} & \textbf{N=500} & \textbf{N=1000} & \textbf{N=2000} & \textbf{N=3000} \\
\midrule
\textbf{MM} & \ensuremath{-98.06^{\scriptscriptstyle -95.57}_{\scriptscriptstyle -98.37}} & \ensuremath{-191.33^{\scriptscriptstyle -185.40}_{\scriptscriptstyle -192.71}} & \ensuremath{-297.04^{\scriptscriptstyle -285.09}_{\scriptscriptstyle -297.94}} & \ensuremath{-420.40^{\scriptscriptstyle -401.43}_{\scriptscriptstyle -421.32}} & \ensuremath{-490.63^{\scriptscriptstyle -467.06}_{\scriptscriptstyle -491.55}} \\
\midrule
\textbf{LR+D r=32} & \ensuremath{-134.95^{\scriptscriptstyle -128.61}_{\scriptscriptstyle -135.55}} & \ensuremath{-264.13^{\scriptscriptstyle -250.62}_{\scriptscriptstyle -266.96}} & \ensuremath{-394.02^{\scriptscriptstyle -370.97}_{\scriptscriptstyle -396.16}} & \ensuremath{-528.06^{\scriptscriptstyle -495.67}_{\scriptscriptstyle -530.21}} & \ensuremath{-598.72^{\scriptscriptstyle -561.71}_{\scriptscriptstyle -601.60}} \\
\textbf{LR+D r=r\textsubscript{max}} & \ensuremath{-98.24^{\scriptscriptstyle -95.97}_{\scriptscriptstyle -98.54}} & \ensuremath{-192.29^{\scriptscriptstyle -186.56}_{\scriptscriptstyle -193.32}} & \ensuremath{-299.63^{\scriptscriptstyle -287.62}_{\scriptscriptstyle -301.14}} & \ensuremath{-425.11^{\scriptscriptstyle -406.19}_{\scriptscriptstyle -426.03}} & \ensuremath{-497.28^{\scriptscriptstyle -474.39}_{\scriptscriptstyle -499.02}} \\
\midrule
\textbf{KS-DPP d=32} & \ensuremath{-99.12^{\scriptscriptstyle -96.17}_{\scriptscriptstyle -99.40}} & \ensuremath{-193.79^{\scriptscriptstyle -186.30}_{\scriptscriptstyle -194.33}} & \ensuremath{-299.17^{\scriptscriptstyle -286.61}_{\scriptscriptstyle -300.81}} & \ensuremath{-423.00^{\scriptscriptstyle -402.65}_{\scriptscriptstyle -424.42}} & \ensuremath{-493.22^{\scriptscriptstyle -468.52}_{\scriptscriptstyle -494.52}} \\
\bottomrule
\end{tabular}
\end{threeparttable}
\end{table*}

\endgroup
To assess the approximation capacity of LR+D and KS-DPP, we compare their test mean log-likelihoods against those obtained by the exact MM-based method \citep{kawashima2023} on the \ttM dataset for small $N$.
Table \ref{tab:comparison_mm} summarizes the principal results, where $r_{\max}=\min\{N, \num{1024}\}$ for LR+D.
All objective values here are evaluated exactly without approximations.
See Appendix \ref{app:exp} for the full table.

Overall, both LR+D and KS-DPP perform slightly below the exact MM.
Although LR+D with $r = N$ contains the full DPP model (by ignoring the small diagonal term), its final log-likelihoods fail to reach those of the exact MM.
This limitation suggests that the structural parameterization of LR+D is not necessarily suited for stochastic optimization.
Moreover, this optimization difficulty becomes more pronounced as $N$ increases; our proposed method surpasses LR+D for $N \geq \num{1000}$, even though KS-DPP does not contain the full DPP model and its optimization inherently relies on the Nystr{\"o}m approximation.

\subsection{Measured Computational Cost}\label{sec:measurement}
\begin{figure*}
    \centering
    \captionsetup{skip=0pt}
    \begin{minipage}{0.51\textwidth}
        \includegraphics[width=\columnwidth]{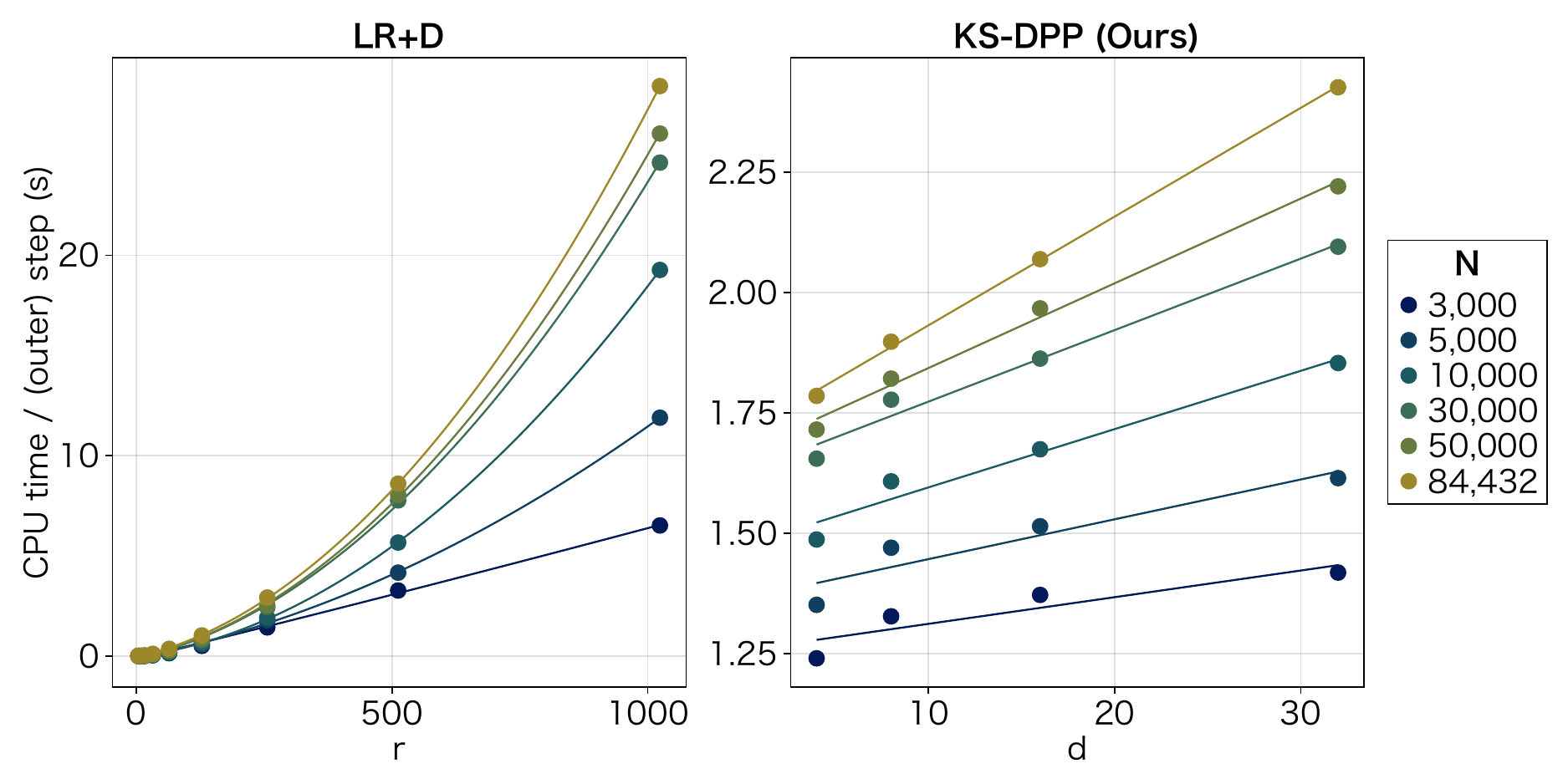}
        \caption{$r$ or $d$ vs. per-(outer-)step computational costs on \ttM. Each marker shows the mean of $3$ trials and the lines are quadratic (left) or linear (right) fits.} 
        \label{fig:cost_r_d}
    \end{minipage}
    \hspace{0.03\textwidth}
    \begin{minipage}{0.43\textwidth}
        \includegraphics[width=\columnwidth]{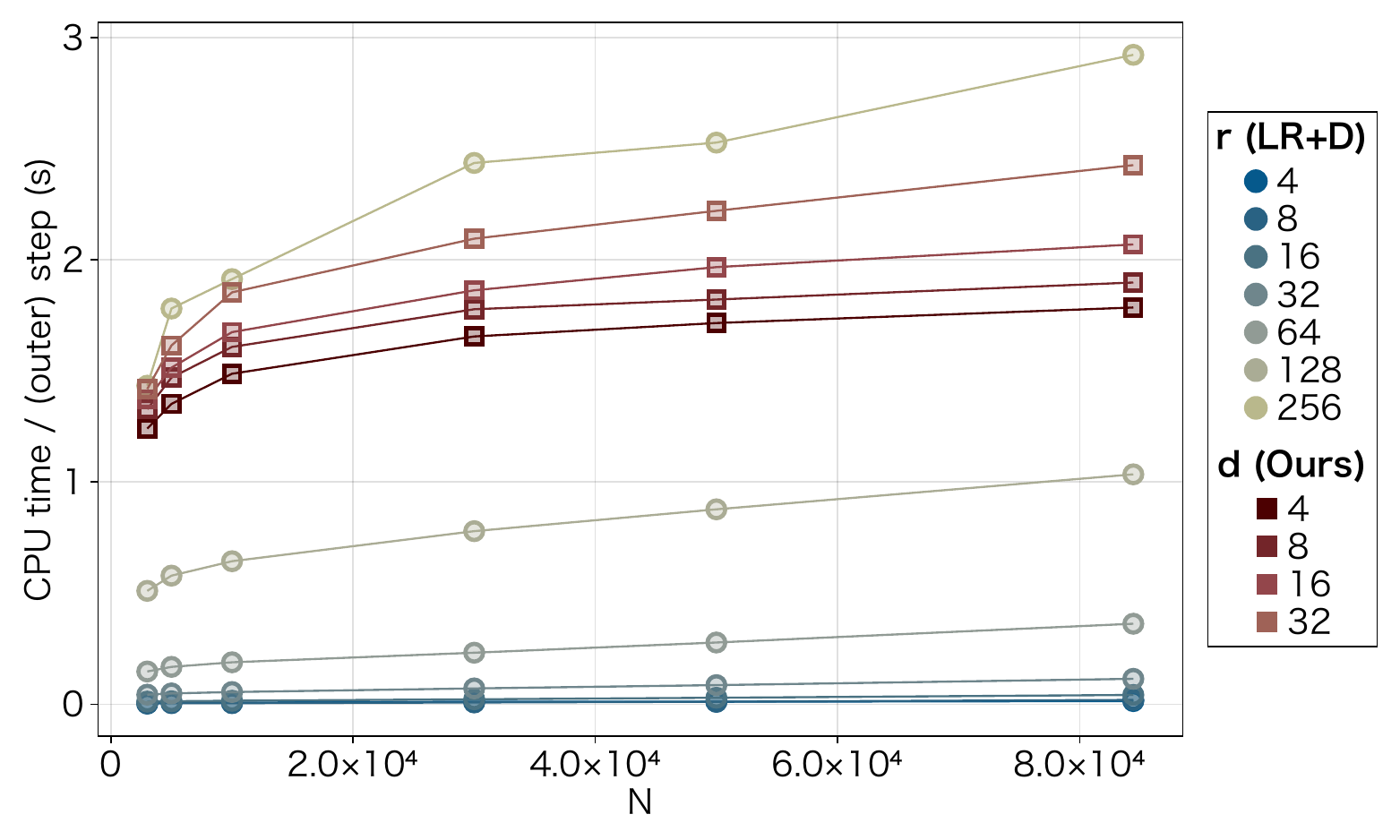}
        \caption{$N$ vs. per-(outer-)step computational costs on \ttM. Each marker shows the mean of $3$ trials.}
        \label{fig:cost_N}
    \end{minipage}
\end{figure*}
We measure the practical computational costs of LR+D and the proposed method across varying $r, d$, and $N$.
We focus here on the \ttM dataset; consistent scaling results for \ttI are provided in Appendix \ref{app:measurement}.
Figure \ref{fig:cost_r_d} illustrates how the per-step and per-outer-step computation times scale with $r$ and $d$ on the \ttM dataset.
As predicted by the complexity analysis in Section \ref{sec:complexity}, the per-outer-step computational cost of KS-DPP increases linearly with $d$, whereas the per-step cost of LR+D increases quadratically with $r$.
For datasets with complex structures and high intrinsic dimensionality, LR+D necessitates a large $r$ to capture such structures; however, the rapidly growing per-step cost can make learning high-quality models computationally prohibitive.
In contrast, the empirical linear coefficient for $d$ in KS-DPP remains favorably small, highlighting its suitability for complex data.
Although KS-DPP requires caching of the Nystr{\"o}m approximation, which incurs a large constant overhead per outer step, this cost is effectively amortized over the inner optimization steps.

Figure \ref{fig:cost_N} shows the dependence of the computational costs on $N$, omitting the results for LR+D with $r \geq 512$ for visual clarity.
Although the theoretical per-(outer-)step costs of both methods are linear with respect to $N$, this linear scaling is not explicitly visible for $N < \num{30000}$.
In this regime, the overall runtime for both methods is dominated by the observation-dependent terms, which are independent of $N$.
For KS-DPP, this overhead corresponds to the $O(T_{\mathrm{in}}B\kappa d T_{\mathrm{solve}})$ term in \eqref{eq:complexity}.

\section{Conclusion}
In this paper, we proposed KS-DPP, a parameterization of the DPP kernel using $(d-1)$-dimensional latent features and a nonlinear kernel function.
The nonlinear parameterization captures full-rank interactions even with a small $d$, achieving remarkably high parameter efficiency.
To learn KS-DPPs at scale, we derived a local quadratic minorizer for the DPP log-likelihood inspired by the MM algorithm, and developed a stochastic solver that evaluates the resulting large linear systems in linear time.
We further employed the quality-weighted Nystr{\"o}m approximation, establishing a theoretical error bound for the bias induced by this approximation.
Empirical evaluations on two real-world datasets demonstrated that KS-DPP significantly improves the trade-off between generalization performance and learning time, particularly as the problem size $N$ and dataset complexity grow.

Future work includes the development of precise and efficient stochastic operations for KS-DPPs, such as evaluation, sampling, and mode exploration.
Particularly for sampling, the Nystr{\"o}m-based dual representation studied by
\citet{affandi2013} may provide a natural starting point for developing
scalable samplers for learned KS-DPPs.

\subsubsection*{Acknowledgements}
Part of this work is supported by JST CREST JPMJCR26E3, JSPS KAKENHI JP25H01494, JP26K02989 and JP26K23861.

%\subsubsection*{References}
\bibliographystyle{apalike}
\bibliography{KSDPP}

\clearpage
\appendix
\thispagestyle{empty}
\counterwithin{figure}{section}
\counterwithin{table}{section}

% Supplementary material: To improve readability, you must use a single-column format for the supplementary material.
\onecolumn

\aistatstitle{Scalable Learning of Determinantal Point Processes with Kernel-Structured Interactions:\\
Supplementary Materials}

\section{Gradient Estimator and Implementation Details}\label{app:gradient_estimators}
\subsection{Surrogate Gradient}
To express the gradient estimators simply, we define some matrices:
\begin{align}
    \bm{Q} &\coloneq \mathrm{diag}(q_1, \ldots, q_N),\\
    \bm{\Psi}_{\bm{\omega}} &\coloneq (\bm{\psi}_{\bm{\omega}}(\bm{x}_1), \ldots, \bm{\psi}_{\bm{\omega}}(\bm{x}_N))^\top,\\
    \bm{\Psi}_{\bm{\omega}, \cA} &\coloneq (\bm{\psi}_{\bm{\omega}}(\bm{x}_i))^\top_{i \in \cA},\\
    \bm{\Delta}^{t,m}_{\bm{\theta}, \bm{\omega}} &\coloneq 
    \bm{Q}[\cA_m]\bm{\Psi}_{\bm{\omega}, \cA_m} \bm{\Psi}^\top_{\bm{\omega}, \cA_m} \bm{Q}[\cA_m]
    - \bm{Q}_t[\cA_m]\bm{\Psi}^t_{\bm{\omega}, \cA_m} \bm{\Psi}^{t\top}_{\bm{\omega}, \cA_m} \bm{Q}_t[\cA_m],
\end{align}
and let $r_m(i)$ be the local index (or rank) of element $i$ within $\cA_m$ when sorted in ascending order.
The sub/super-scripted $t$ of $\bm{Q}_t$ and $\bm{\Psi}^t_{\bm{\omega}}$ indicates that they are parameterized by the previous parameter $\bm{\theta}_t$.
Note that $\bm{\Psi}_{\bm{\omega}}$ and $\bm{\Psi}^t_{\bm{\omega}}$ are usually tall because the random feature $\bm{\psi}_{\bm{\omega}}(\bm{x})$ is low-dimensional in many cases.
For example, if the kernel function $k$ is shift-invariant and $\bm{\psi}$ is the RFF, $\bm{\psi}_{\bm{\omega}}(\bm{x}) = (\cos(\bm{\omega}^\top \bm{x}), \sin(\bm{\omega}^\top \bm{x}))^\top \in \mathbb{R}^2$.

The surrogate \eqref{eq:surrogate} involves two types of terms about $\bm{L}_{\bm{\theta}}$: linear and quadratic terms, except for the regularizer.
Since $\bm{L}_{\bm{\theta}}$ is defined as \eqref{eq:l_structure}, and the kernel function has the feature map representation \eqref{eq:feature_map},
the derivative of the third term of the surrogate becomes:
\begin{align}
    \label{eq:derivative_denom}
    \begin{aligned}
        \nabla_{q_i} \tr\{(\bm{I}+\bm{L}_t)^{-1} \bm{L}_{\bm{\theta}}\}
        & = 2 \mathbb{E}_{\bm{\omega}}[\bm{\psi}_{\bm{\omega}}(\bm{x}_i)^\top [\bm{F}^t (\bm{\omega})]_{i, :}],\\
        \nabla_{x_{i,\ell}} \tr\{(\bm{I}+\bm{L}_t)^{-1} \bm{L}_{\bm{\theta}}\}
        & = 2 q_i \mathbb{E}_{\bm{\omega}} \left [ \left (\frac{\partial \bm{\psi}_{\bm{\omega}}(\bm{x}_i)}{\partial x_{i,\ell}} \right )^\top [\bm{F}^t(\bm{\omega})]_{i, :} \right ],
    \end{aligned}
\end{align}
where
\begin{align}
    \bm{F}^t(\bm{\omega}) \coloneq (\bm{I} + \bm{L}_t)^{-1} \bm{Q} \bm{\Psi}_{\bm{\omega}}
\end{align}
and $[\cdot]_{i, :}$ denotes the $i$-th row of the matrix inside.

By the random features with independent parameters $\bm{\omega}, \bm{\nu} \sim p$,
the gradients of the second term of \eqref{eq:surrogate} are given by:
\begin{align}
    \label{eq:derivative_numer_quad}
    \begin{aligned}
        &\nabla_{q_i} \left ( -\frac{\beta_\rho}{2} \lVert \bm{E}_{t,m} \rVert^2_F\right )
        = \left \{
        \begin{aligned}
            &-2\beta_\rho \mathbb{E}_{\bm{\omega}, \bm{\nu}}[
            \bm{\psi}_{\bm{\omega}}(\bm{x}_i)^\top
            [\bm{G}^{t,m}_{\bm{\theta}}(\bm{\omega}, \bm{\nu})]_{r_m(i), :}
            ] && (\mbox{for $i \in \cA_m$}),\\
            &\quad 0 && (\mbox{for $i \notin \cA_m$}),
        \end{aligned}
        \right .
        \\
        &\nabla_{x_{i,\ell}} \left ( -\frac{\beta_\rho}{2} \lVert \bm{E}_{t,m} \rVert^2_F\right )
        = \left \{
        \begin{aligned}
            &-2\beta_\rho q_i \mathbb{E}_{\bm{\omega}, \bm{\nu}}\left [
            \left (\frac{\partial\bm{\psi}_{\bm{\omega}}(\bm{x}_i)}{\partial x_{i, \ell}} \right )^\top
            [\bm{G}^{t,m}_{\bm{\theta}}(\bm{\omega}, \bm{\nu})]_{r_m(i), :}
            \right ] && (\mbox{for $i \in \cA_m$}),\\
            &\quad 0 && (\mbox{for $i \notin \cA_m$}),
        \end{aligned}
        \right .
    \end{aligned}
\end{align}
where
\begin{align}
    \bm{G}^{t,m}_{\bm{\theta}}(\bm{\omega}, \bm{\nu})
    \coloneq \bm{L}_t[\cA_m]^{-1}\bm{\Delta}^{t,m}_{\bm{\theta}, \bm{\nu}} \bm{L}_t[\cA_m]^{-1}\bm{Q}[\cA_m] \bm{\Psi}_{\bm{\omega}, \cA_m}.
\end{align}

Similarly, we have
\begin{align}
    \label{eq:derivative_numer_linear}
    \begin{aligned}
        &\nabla_{q_i} \tr \bm{E}_{t,m}
        = \left \{
        \begin{aligned}
            &~2 \mathbb{E}_{\bm{\omega}}[\bm{\psi}_{\bm{\omega}}(\bm{x}_i)^\top [\bm{H}^{t,m}(\bm{\omega})]_{r_m(i), :}] && (\mbox{for $i \in \cA_m$}),\\
            &\quad 0 && (\mbox{for $i \notin \cA_m$}),
        \end{aligned}
        \right .\\
        &\nabla_{x_{i,\ell}} \tr \bm{E}_{t,m}
        = \left \{
        \begin{aligned}
            &~2 q_i \mathbb{E}_{\bm{\omega}}\left [ \left (\frac{ \partial \bm{\psi}_{\bm{\omega}}(\bm{x}_i)}{\partial x_{i,\ell}}\right )^\top [\bm{H}^{t,m}(\bm{\omega})]_{r_m(i), :} \right] && (\mbox{for $i \in \cA_m$}),\\
            &\quad 0 && (\mbox{for $i \notin \cA_m$}),
        \end{aligned}
        \right .
    \end{aligned}
\end{align}
for the gradients of the first term of \eqref{eq:surrogate}, where
\begin{align}
    \bm{H}^{t,m}(\bm{\omega}) \coloneq \bm{L}_t[\cA_m]^{-1} \bm{Q}[\cA_m] \bm{\Psi}_{\bm{\omega}, \cA_m}.
\end{align}

\section{Missing Proofs}\label{app:proof}
For notational simplicity, we assume that the $t$-th parameter $\bm{\theta}_t$ is already fixed and omit the iteration index $t$ in this section.

\subsection{Proof of Lemma \ref{lem:gap_bound}}\label{app:proof_lemma}
\begin{proof}
    We define $\widehat{\bm{S}} \coloneq (\bm{I}+\widehat{\bm{L}})^{-1}$ and $\bm{\Delta} \coloneq \bm{L} - \widehat{\bm{L}}$.
    Then, we have
    \begin{align}
        f_{\mathrm{den}}(\bm{L})
        &= - \log\det(\bm{I} + \bm{L})\\
        &= - \log\det(\bm{I}+\widehat{\bm{L}}+\bm{\Delta} )\\
        &= - \log\det((\bm{I}+\widehat{\bm{L}})(\bm{I}+\widehat{\bm{S}}\bm{\Delta}))\\
        &= - \log\det(\bm{I}+\widehat{\bm{L}}) - \log\det(\bm{I}+\widehat{\bm{S}}\bm{\Delta})
    \end{align}
    and the gap $\gamma$ is written as
    \begin{align}
        \gamma
        = \ell_{\lambda}(\bm{L}) - \hat{g}_{\lambda}(\bm{L}; \bm{L}, \widehat{\bm{L}})
        = -\log\det(\bm{I}+\widehat{\bm{S}}\bm{\Delta})+ \tr(\widehat{\bm{S}}\bm\Delta)
        = \sum^N_{i=1} ( \mu_i - \log (1 + \mu_i)),
    \end{align}
    where $\mu_i$ denotes the $i$-th eigenvalue of $\widehat{\bm{S}}^{\frac{1}{2}} \bm{\Delta} \widehat{\bm{S}}^{\frac{1}{2}}$.
    Note that $\bm{\Delta} \succeq 0$ and $\mu_i \geq 0$ always holds from the inequalities $0 \preceq \widehat{\bm{L}} \preceq \bm{L}$ because the Nystr{\"o}m approximation is induced by an orthogonal projection.
    Since
    \begin{align}
        x - \log(1+x) = \int^x_0 \frac{u}{1+u} du \leq \int^x_0 u du = \frac{x^2}{2},
    \end{align}
    holds for $x \geq 0$, 
    \begin{align}
        \gamma
        = \sum^N_{i=1} ( \mu_i - \log (1 + \mu_i)) \leq \sum^N_{i=1}\frac{\mu_i^2}{2}
        = \frac{1}{2} \lVert \widehat{\bm{S}}^{\frac{1}{2}} \bm{\Delta} \widehat{\bm{S}}^{\frac{1}{2}}\rVert^2_F
        = \frac{1}{2} \lVert (\bm{I} + \widehat{\bm{L}})^{-\frac{1}{2}} (\bm{L} - \widehat{\bm{L}}) (\bm{I} + \widehat{\bm{L}})^{-\frac{1}{2}}\rVert^2_F.
    \end{align}
\end{proof}

\subsection{Proof of Theorem \ref{thm:ex_gap_bound}}
In Section \ref{sec:methodology}, we introduced the quality-weighted Nystr{\"o}m method, in which $\bm{M} = \bm{Q} \bm{\Sigma} \bm{Q}$ is approximated as
$\bm{M} \approx \widehat{\bm{M}} = \bm{M}^\top_{\cJ, :}\bm{M}[\cJ]^+ \bm{M}_{\cJ, :}$.
Here, the indices of $\cY$ are sampled by following the probabilities $p_j$ without replacement, and $\cJ = \{j_1, \ldots, j_{r_{\mathrm{Nys}}}\}\subseteq \cY$ is the obtained set of sampled indices.
We first check that the bound \eqref{eq:ex_gap_bound} is valid for sampling with replacement, and then show that the same bound also applies to sampling without replacement.

\begin{lemma}\label{lem:ex_gap_bound_wr}
    Suppose that the $N \times N$ matrix $\bm{\Sigma}$ is positive definite and the diagonal terms are $\Sigma_{ii} = k_0$,
    $\bm{Q} = \mathrm{diag}(q_1, \ldots, q_N) \succeq 0$,
    and $\bm{L} = \varepsilon \bm{I} + \bm{Q}\bm{\Sigma}\bm{Q} = \varepsilon \bm{I} + \bm{M}$.
    Let $p_j = q^2_j / \lVert \bm{Q} \rVert^2_F$ be the sampling probability of $j \in \cY$
    and $\cJ = \{j_1, \ldots, j_r\} \subseteq \cY$ be the sampled set,
    where $j_i$ is the $i$-th independent sample from $\cY = \{1, \ldots, N\}$ with replacement.
    In addition, we denote the Nystr{\"o}m approximation of $\bm{M}$ under $\cJ$ as $\widehat{\bm{M}}$ and $\widehat{\bm{L}} \coloneq \varepsilon \bm{I} + \widehat{\bm{M}}$.

    Then, the expected gap admits
    \begin{align}
        \mathbb{E}[\gamma] \leq \frac{k^2_0 \lVert \bm{Q} \rVert^4_F}{2(1+\varepsilon)^2 r}.
    \end{align}
\end{lemma}
\begin{proof}
    Since $\bm{M} \succeq 0$, there exist $m \in \mathbb{N}$ and $\bm{V} \in \mathbb{R}^{m \times N}$ such that $\bm{M} = \bm{V}^\top\bm{V}$.
    Let $\bm{C}$ be the dual of $\bm{M}$, i.e., $\bm{C} = \bm{V}\bm{V}^\top \in \mathbb{R}^{m \times m} = \sum^N_{i=1} \bm{v}_i \bm{v}^\top_i$, where $\bm{v}_i$ is the $i$-th column vector of $\bm{V}$.
    We additionally introduce $\bm{Y}_s \coloneq \bm{v}_{j_s} \bm{v}_{j_s}^\top / p_{j_s}$ and $\widehat{\bm{C}} \coloneq r^{-1} \sum^r_{s=1} \bm{Y}_s$ for the $s$-th sampled index $j_s$.
    Clearly, we have $\mathbb{E}[\bm{Y}_s] = \mathbb{E}[\widehat{\bm{C}}] = \bm{C}$.
    We then have
    \begin{align}
        \mathbb{E} \lVert \bm{C} - \widehat{\bm{C}} \rVert^2_F
        &= \mathbb{E} \left \lVert \frac{1}{r} \sum^r_{s=1} (\bm{C} - \bm{Y}_s) \right \rVert^2_F \\
        &= \frac{1}{r^2} \sum^r_{s=1} \mathbb{E} \lVert \bm{C} - \bm{Y}_s \rVert^2_F\\
        &= \frac{1}{r} \mathbb{E} \lVert \bm{C} - \bm{Y}_1 \rVert^2_F\\
        &= \frac{1}{r} (\mathbb{E}\lVert\bm{Y}_1 \rVert^2_F - \lVert \bm{C} \rVert^2_F)\\
        &\leq \frac{1}{r} \mathbb{E}\lVert\bm{Y}_1 \rVert^2_F\\
        &= \frac{1}{r} \sum^N_{i=1} \frac{\lVert \bm{v}_i \rVert^4}{p_i}\\
        \label{eq:proof_lemma_1}
        &= \frac{k^2_0 \lVert \bm{Q} \rVert^4_F}{r},
    \end{align}
    where we used $M_{jj} = \lVert \bm{v}_j \rVert^2 = k_0 q_j^2$, which implies $\sum_i \lVert \bm{v}_i \rVert^2 = k_0 \lVert \bm{Q} \rVert^2_F$ and $p_j = \lVert \bm{v}_j \rVert^2 / \sum_i \lVert\bm{v}_i\rVert^2$.

    Let $\bm{P}_{\cJ}$ be the projection matrix of $\mathrm{span}(\{\bm{v}_j\}_{j \in \cJ})$ and $\bm{\Pi}_{\cJ}$ be the projection matrix of the orthogonal complement: $\bm{P}_{\cJ} \coloneq \bm{V}_{:, \cJ} (\bm{V}^\top_{:, \cJ} \bm{V}_{:, \cJ})^+ \bm{V}^\top_{:, \cJ}$ and $\bm{\Pi}_{\cJ} \coloneq \bm{I} - \bm{P}_{\cJ}$.
    Then, $\bm{V}^\top \bm{P}_{\cJ} \bm{V} = \widehat{\bm{M}}$ and
    \begin{align}
        \bm{V}^\top \bm{\Pi}_{\cJ} \bm{V} = \bm{V}^\top \bm{V} - \bm{V}^\top\bm{P}_{\cJ}\bm{V} = \bm{M} - \widehat{\bm{M}} =\bm{\Delta}.
    \end{align}
    This admits
    \begin{align}
        \lVert \bm{\Delta} \rVert^2_F = \tr \{ \bm{V}^\top \bm{\Pi}_{\cJ} \bm{V}\bm{V}^\top \bm{\Pi}_{\cJ} \bm{V}\}
        = \tr \{ \bm{V}^\top \bm{\Pi}_{\cJ} \bm{\Pi}_{\cJ} \bm{V}\bm{V}^\top \bm{\Pi}_{\cJ} \bm{\Pi}_{\cJ} \bm{V}\}
        = \tr \{ \bm{\Pi}_{\cJ} \bm{C} \bm{\Pi}_{\cJ}\bm{\Pi}_{\cJ} \bm{C} \bm{\Pi}_{\cJ}\}
        = \lVert \bm{\Pi}_{\cJ} \bm{C} \bm{\Pi}_{\cJ}\rVert^2_F.
    \end{align}
    Thus, we have
    \begin{align}
        \label{eq:proof_lemma_2}
        \lVert \bm{\Delta} \rVert^2_F = \lVert \bm{\Pi}_{\cJ} \bm{C} \bm{\Pi}_{\cJ}\rVert^2_F
        = \lVert \bm{\Pi}_{\cJ} (\bm{C} - \widehat{\bm{C}}) \bm{\Pi}_{\cJ}\rVert^2_F
        \leq \lVert \bm{\Pi}_{\cJ} \rVert^4_{\mathrm{op}} \lVert \bm{C} - \widehat{\bm{C}} \rVert^2_F
        = \lVert \bm{C} - \widehat{\bm{C}} \rVert^2_F,
    \end{align}
    where we used $\bm{\Pi}_{\cJ} \widehat{\bm{C}} \bm{\Pi}_{\cJ} = \bm{O}$ and the fact that all the eigenvalues of a projection matrix are $0$ or $1$.

    According to the proof of Lemma \ref{lem:gap_bound} (i.e., Section \ref{app:proof_lemma}), we define $\widehat{\bm{S}} \coloneq (\bm{I}+\widehat{\bm{L}})^{-1}$ and $\bm{\Delta} \coloneq \bm{L} - \widehat{\bm{L}} \succeq 0$.
    From Lemma \ref{lem:gap_bound} and $0 \preceq \widehat{\bm{S}}^{\frac{1}{2}} = (\bm{I} + \widehat{\bm{L}})^{-\frac{1}{2}} \preceq (1 + \varepsilon)^{-\frac{1}{2}} \bm{I}$,
    we have
    \begin{align}
        \gamma \leq \frac{1}{2} \lVert \widehat{\bm{S}}^{\frac{1}{2}} \bm{\Delta} \widehat{\bm{S}}^{\frac{1}{2}}\rVert^2_F
        \leq \frac{1}{2} \lVert \widehat{\bm{S}}^{\frac{1}{2}} \rVert^4_{\mathrm{op}} \lVert \bm{\Delta} \rVert^2_F
        \leq \frac{\lVert \bm{\Delta} \rVert^2_F}{2 (1 + \varepsilon)^2},
    \end{align}
    where we used $\lVert \bm{A} \rVert^2_F \leq \lVert \bm{B} \rVert^2_F $ for any $0 \preceq \bm{A} \preceq \bm{B}$ for the final inequality,
    because $\lVert \bm{B} \rVert^2_F - \lVert \bm{A} \rVert^2_F = \tr(\bm{B}^2 - \bm{A}^2) = \tr\{(\bm{B} - \bm{A})(\bm{B} + \bm{A})\} \geq 0$.
    Thus, from \eqref{eq:proof_lemma_1} and \eqref{eq:proof_lemma_2},
    \begin{align}
        \mathbb{E}[\gamma] \leq \frac{\mathbb{E} \lVert \bm{\Delta} \rVert^2_F}{2 (1 + \varepsilon)^2}
        \leq \frac{\mathbb{E} \lVert \bm{C} - \widehat{\bm{C}} \rVert^2_F}{2 (1 + \varepsilon)^2}
        \leq \frac{k^2_0 \lVert \bm{Q} \rVert^4_F}{2 (1 + \varepsilon)^2 r},
    \end{align}
    and the lemma is confirmed.
\end{proof}

We now consider sampling without replacement.
Mathematically, weighted sampling without replacement from $\cY$ is realized in the following way \citep{ben-hamou2018}.

\begin{minipage}{\textwidth}
    \begin{tcolorbox}[standard jigsaw, opacityback=0]
        \begin{enumerate}
            \item Repeat sampling with replacement from $\cY$ according to the probabilities $p_1, \ldots, p_N$, and make an infinitely long tuple $\mathcal{I} \coloneq (i_1, i_2, \ldots)$.
            \item Create $\cJ = \{j_1, \ldots, j_r\}$ by extracting $r$ distinct indices from $\mathcal{I}$, excluding those already chosen from the beginning.
        \end{enumerate}
    \end{tcolorbox}
    \captionof{proc}{Weighted sampling without replacement.}
    \label{proc:wor}
\end{minipage}
This procedure yields the conditional probability
\begin{align}
    \label{eq:wor_conditional_prob}
    P(j_i ~|~ j_1, \ldots, j_{i-1}) = \frac{p_{j_i}}{1 - \sum_{\ell < i} p_{j_\ell}} \quad (i \notin \{j_1, \ldots, j_{i-1}\}).
\end{align}

\begin{theorem}[Restatement of Theorem \ref{thm:ex_gap_bound}]\label{thm:ex_gap_bound_wor}
    Consider the quality-weighted Nystr{\"o}m method with sampling without replacement, and assume all other settings are the same as in Lemma \ref{lem:ex_gap_bound_wr}.
    Then, we have the same bound \eqref{eq:ex_gap_bound} for the expected anchor gap $\mathbb{E}[\gamma]$.
\end{theorem}
\begin{proof}
    From Procedure \ref{proc:wor}, we obtain $\cJ$ and define $\mathcal{I}_r$ as an $r$-tuple with the first $r$ elements of $\mathcal{I}$.
    We use $\mathcal{I}_r$ and $\cJ$ for the Nystr{\"o}m approximations with and without replacement, respectively.
    Let us define $\widehat{\bm{M}}_{\mathrm{wr}} \coloneq \bm{M}^\top_{\mathcal{I}_r, :}\bm{M}[\mathcal{I}_r]^+ \bm{M}_{\mathcal{I}_r, :}$ as the approximated kernel matrix obtained by sampling with replacement, $\widehat{\bm{M}}_{\mathrm{wor}} \coloneq \bm{M}^\top_{\cJ, :}\bm{M}[\cJ]^+ \bm{M}_{\cJ, :}$ as the one obtained via sampling without replacement, and $\bm{\Delta}_{\mathrm{wr}} \coloneq \bm{M} - \widehat{\bm{M}}_{\mathrm{wr}}, \bm{\Delta}_{\mathrm{wor}} \coloneq \bm{M} - \widehat{\bm{M}}_{\mathrm{wor}}$.
    Note that $\widehat{\bm{M}}_{\mathrm{wr}}$ is invariant to the existence of duplicated indices in $\mathcal{I}_r$ because they do not expand the range of the underlying projection matrix.

    Since the set of distinct elements of $\mathcal{I}_r$ is a subset of $\cJ$, $\widehat{\bm{M}}_{\mathrm{wr}} \preceq \widehat{\bm{M}}_{\mathrm{wor}}$ holds, and $\bm{\Delta}_{\mathrm{wor}} \preceq \bm{\Delta}_{\mathrm{wr}}$ follows.
    Recalling Lemma \ref{lem:ex_gap_bound_wr} and that the quadratic bound of the anchor gap \eqref{eq:gap_bound} holds regardless of the sampling procedure,
    we have
    \begin{align}
        \mathbb{E}[\gamma]
        \leq \frac{\mathbb{E} \lVert \bm{\Delta}_{\mathrm{wor}} \rVert^2_F}{2 (1 + \varepsilon)^2}
        \leq \frac{\mathbb{E} \lVert \bm{\Delta}_{\mathrm{wr}} \rVert^2_F}{2 (1 + \varepsilon)^2}
        \leq \frac{k^2_0 \lVert \bm{Q} \rVert^4_F}{2 (1 + \varepsilon)^2 r},
    \end{align}
\end{proof}
Note that the same conditional distribution \eqref{eq:wor_conditional_prob} is also realized by the Gumbel top-$r$ method \citep{kool2019}, and it may be more suitable for practical use.

\section{Experimental Configurations and Additional Results}\label{app:exp}
\subsection{Detailed Experimental Settings}\label{app:exp_details}
This section provides the detailed configurations omitted from Section \ref{sec:ex}.

\textbf{Datasets:} The full \ttM dataset contains $\num{87585}$ movies rated by over $\num{200000}$ users, with each user rating at least $20$ movies.
The \ttI dataset comprises $\num{49688}$ grocery products across over $\num{3000000}$ orders from over $\num{200000}$ users.
When sweeping the ground set size $N$ by popularity, the exact sizes are $N \in \{\num{3000}, \num{5000}, \num{10000}, \num{30000}, \num{50000}, N_{\max}\}$ for \ttM, and $N \in \{\num{3000}, \num{5000}, \num{10000}, \num{20000}, N_{\max}\}$ for \ttI.

\textbf{Number of Random Features:} For the Monte Carlo gradient estimation in KS-DPPs, the number of sampled random features is set to $2$ for the observation-dependent terms and $4$ for the term originating from the normalizing constant.

\textbf{Evaluation Procedure:} For large $N$, evaluating exact validation and test mean log-likelihoods is computationally difficult for KS-DPPs.
To ensure the fairness of the comparison, we compute the observation-dependent terms exactly and approximate the observation-independent term $f_{\mathrm{den}}$ using the quality-weighted Nystr{\"o}m method with a sufficiently large rank $r_{\mathrm{Nys}} = 512$.
We empirically observe that the bias induced by the rank-$512$ quality-weighted Nystr{\"o}m approximation is less than $1 \%$ of the exact value for synthetic data.
See the next section on this experiment.

\textbf{Computing Infrastructure:} All experiments were executed in Julia 1.12.6 on an Ubuntu 24.04 system equipped with an Intel Core i9-9980XE CPU ($3.00$ GHz, 18 cores) and $128$ GB of RAM.

\begin{figure*}
    \centering
    \begin{minipage}{0.43\textwidth}
        \includegraphics[width=\columnwidth]{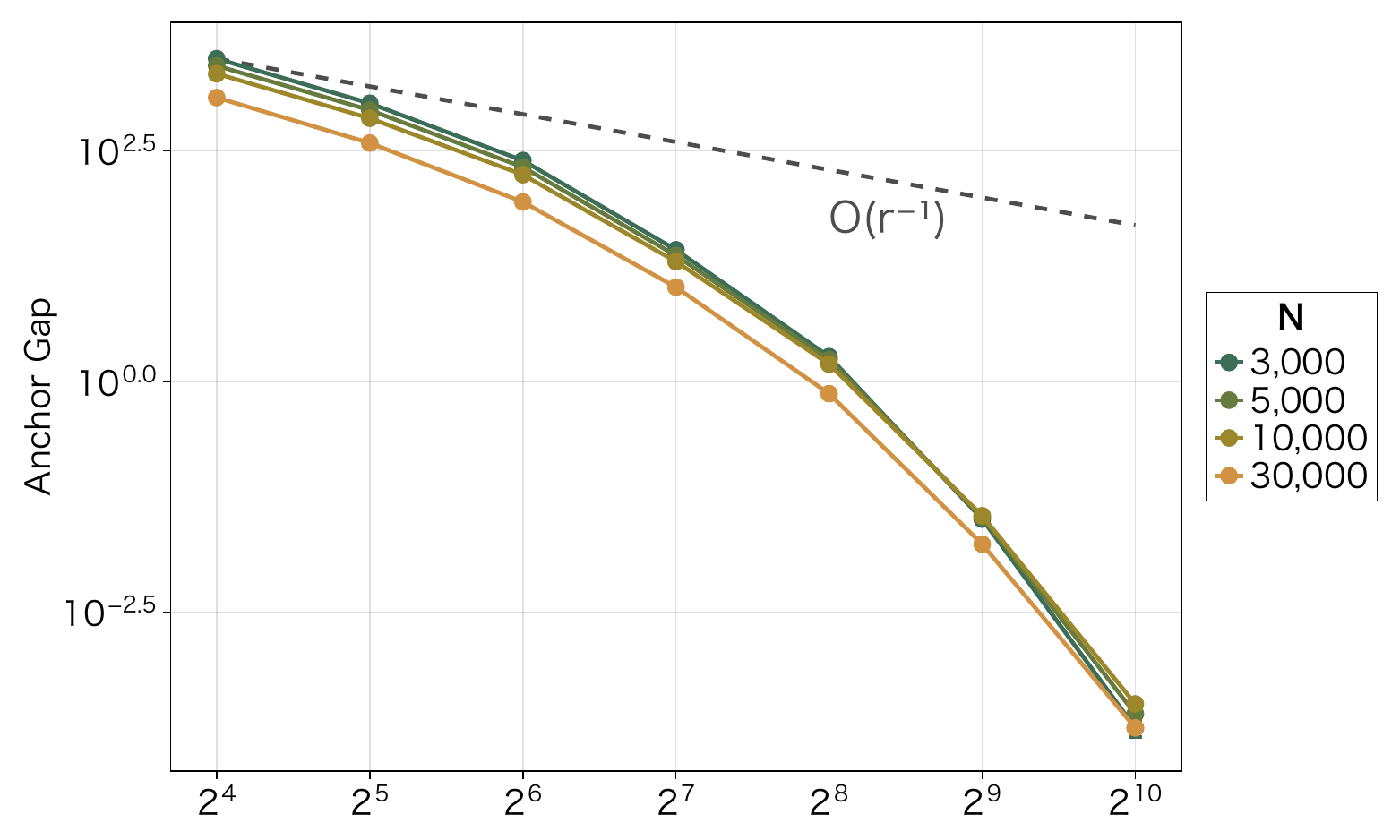}
        \caption{Measurement of the expected anchor gap $\bm{E}[\gamma | \bm{\theta}]$.} 
        \label{fig:anchor_gap}
    \end{minipage}
    \hspace{0.03\textwidth}
    \begin{minipage}{0.51\textwidth}
        \includegraphics[width=\columnwidth]{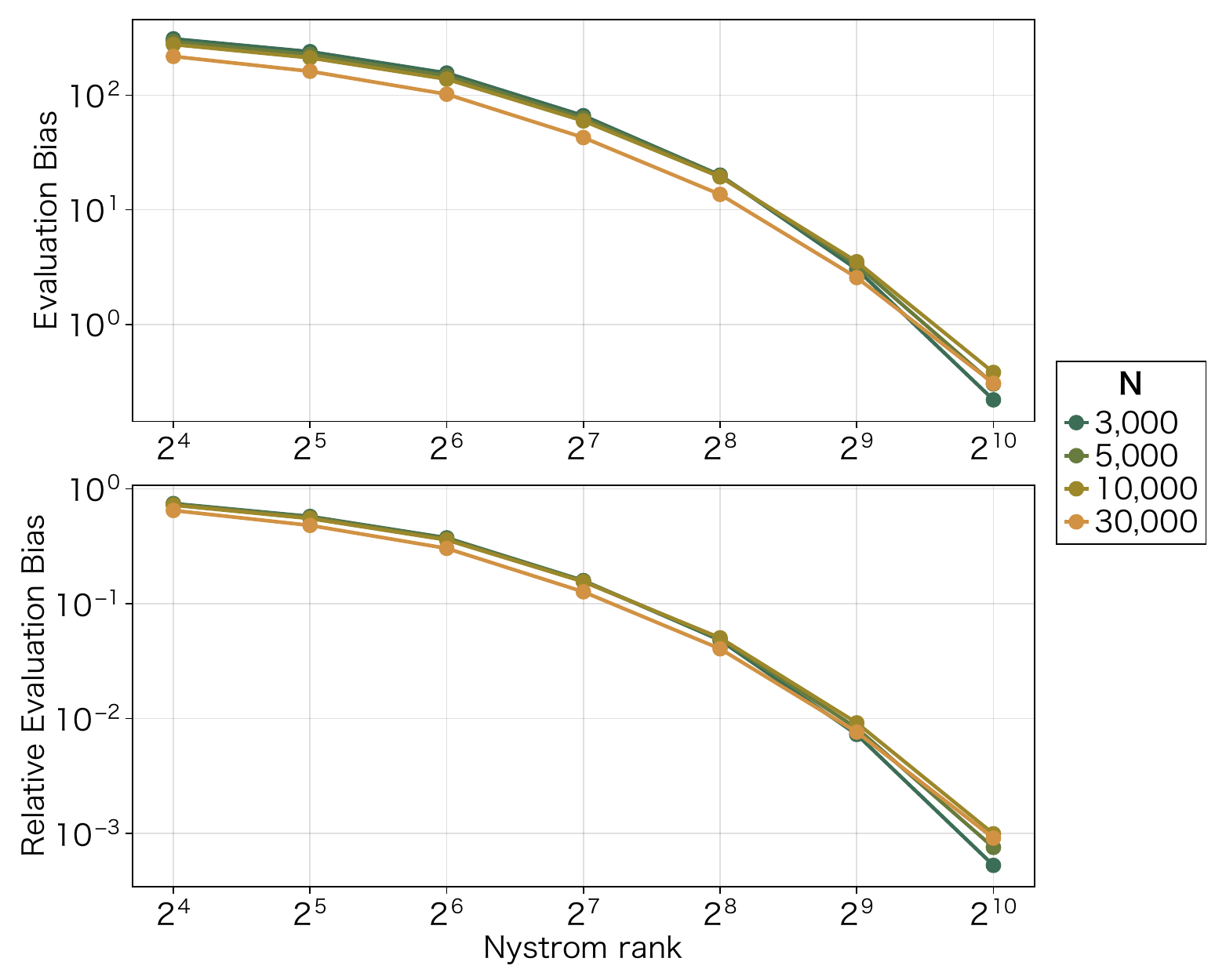}
        \caption{Measurement of the evaluation error of $f_{\mathrm{den}}$.}
        \label{fig:eval_gap}
    \end{minipage}
\end{figure*}
\subsection{Rank Dependency of Quality-Weighted Nystr{\"o}m Approximation}
We empirically evaluate the approximation error induced by the quality-weighted Nystr{\"o}m method to (i) verify the theoretical decay rate of the expected anchor gap (Theorem \ref{thm:ex_gap_bound}) and (ii) confirm that $r_{\mathrm{Nys}} = 512$ is sufficient for the held-out evaluations in Section \ref{sec:ex}.
We generate synthetic data mimicking real-world popularity skew as follows.
First, we set $d=8$, meaning that the latent dimensionality is $d-1=7$, and independently sample $N$ latent features $\bm{x}_1, \ldots, \bm{x}_N \sim \mathcal{N}(\bm{0}, (d-1)^{-1}\bm{I})$.
The kernel matrix $\bm{\Sigma}$ is constructed using these latent features and a Gaussian kernel with its bandwidth determined by the median heuristic.
Next, we sample item qualities $q_1, \ldots, q_N$ from a log-normal distribution such that the expected subset size is $\mathbb{E}[|\cA|] \approx 150$.
Specifically, given i.i.d. standard Gaussian samples $z_1, \ldots, z_N \sim \mathcal{N}(0, 1)$, we define $q_i = \exp (\mu + \sigma z_i)$.
Since $\mathbb{E}[|\cA|]$ increases monotonically with $\mu$ for a fixed $\sigma$, we can identify the desired $\mu$, which realizes $\mathbb{E}[|\cA|] \approx 150$, via binary search.
Fixing $\sigma = 0.75$, we find the appropriate $\mu$ using $5$ pilot sets of $\bm{x}_1, \ldots, \bm{x}_N$ and $z_1, \ldots, z_N$.
For each target rank $r_{\mathrm{Nys}}$, we perform $5$ independent quality-weighted Nystr{\"o}m approximations and compute the exact values of the anchor gap $\gamma$ and the normalizing term $f_{\mathrm{den}}(\bm{L}) = -\log\det(\bm{I} + \bm{L})$.
We repeat this entire procedure $20$ times for each problem size $N \in \{\num{3000}, \num{5000}, \num{10000}, \num{30000}\}$.

Figure \ref{fig:anchor_gap} illustrates the measured decay of the expected anchor gap $\mathbb{E}[\gamma | \bm{\theta}]$.
Error bars representing the standard deviations are plotted but are heavily obscured by the markers due to their small magnitude.
These results yield two key insights.
First, the theoretical decay rate $O(r^{-1})$ is somewhat conservative for this typical synthetic setting.
Second, the decay behavior is largely unaffected by the problem size $N$.
This robustness is predicted by the bound \eqref{eq:ex_gap_bound}, which depends on the sum of item qualities rather than $N$ directly, indicating that the expected subset size under the current model $\bm{L}_t$ (which is largely determined by $\bm{Q}_t$) governs the decay rate more significantly than $N$.

Figure \ref{fig:eval_gap} displays the absolute and relative biases arising from the approximation of $f_{\mathrm{den}}$ during validation and test log-likelihood evaluations.
Again, standard deviation error bars are included but remain invisible due to their negligible scale.
The empirical results demonstrate that the relative bias introduced by the rank-$512$ quality-weighted Nystr{\"o}m approximation strictly remains below $1 \%$ of the exact value across all tested $N$.
This confirms that the approximation biases for KS-DPPs in Section \ref{sec:ex} are safely negligible, validating our empirical comparisons.

\begingroup
\renewcommand{\arraystretch}{1.3}
\begin{table}[!ht]
\setlength\tabcolsep{0pt}
\centering
\caption{Comparison of the test mean log-likelihoods with the exact MM method on \ttM. The values denote $\mbox{med}^{\mbox{max}}_{\mbox{min}}$ of 3 trials. }
\label{tab:comparison_mm_full}
\begin{threeparttable}
\begin{tabular}{@{\extracolsep{2ex}}*{6}{lccccc}}
\toprule
\multicolumn{1}{c}{\textbf{Method}} & \multicolumn{5}{c}{\textbf{N}} \\
\multicolumn{1}{c}{} & \textbf{200} & \textbf{500} & \textbf{1000} & \textbf{2000} & \textbf{3000} \\
\midrule
\textbf{MM} & \ensuremath{-98.06^{\scriptscriptstyle -95.57}_{\scriptscriptstyle -98.37}} & \ensuremath{-191.33^{\scriptscriptstyle -185.40}_{\scriptscriptstyle -192.71}} & \ensuremath{-297.04^{\scriptscriptstyle -285.09}_{\scriptscriptstyle -297.94}} & \ensuremath{-420.40^{\scriptscriptstyle -401.43}_{\scriptscriptstyle -421.32}} & \ensuremath{-490.63^{\scriptscriptstyle -467.06}_{\scriptscriptstyle -491.55}} \\
\midrule
\textbf{LR+D r=4} & \ensuremath{-178.58^{\scriptscriptstyle -171.02}_{\scriptscriptstyle -179.68}} & \ensuremath{-311.40^{\scriptscriptstyle -296.57}_{\scriptscriptstyle -314.58}} & \ensuremath{-442.03^{\scriptscriptstyle -417.56}_{\scriptscriptstyle -444.55}} & \ensuremath{-575.97^{\scriptscriptstyle -542.57}_{\scriptscriptstyle -578.60}} & \ensuremath{-646.88^{\scriptscriptstyle -608.39}_{\scriptscriptstyle -649.53}} \\
\textbf{LR+D r=8} & \ensuremath{-169.19^{\scriptscriptstyle -161.78}_{\scriptscriptstyle -170.39}} & \ensuremath{-301.91^{\scriptscriptstyle -287.26}_{\scriptscriptstyle -305.11}} & \ensuremath{-432.51^{\scriptscriptstyle -408.20}_{\scriptscriptstyle -435.13}} & \ensuremath{-566.57^{\scriptscriptstyle -533.25}_{\scriptscriptstyle -569.30}} & \ensuremath{-637.41^{\scriptscriptstyle -599.12}_{\scriptscriptstyle -640.18}} \\
\textbf{LR+D r=16} & \ensuremath{-154.29^{\scriptscriptstyle -147.25}_{\scriptscriptstyle -155.49}} & \ensuremath{-286.06^{\scriptscriptstyle -271.78}_{\scriptscriptstyle -289.27}} & \ensuremath{-416.61^{\scriptscriptstyle -392.61}_{\scriptscriptstyle -419.19}} & \ensuremath{-550.56^{\scriptscriptstyle -517.61}_{\scriptscriptstyle -553.24}} & \ensuremath{-621.30^{\scriptscriptstyle -583.51}_{\scriptscriptstyle -624.31}} \\
\textbf{LR+D r=32} & \ensuremath{-134.95^{\scriptscriptstyle -128.61}_{\scriptscriptstyle -135.55}} & \ensuremath{-264.13^{\scriptscriptstyle -250.62}_{\scriptscriptstyle -266.96}} & \ensuremath{-394.02^{\scriptscriptstyle -370.97}_{\scriptscriptstyle -396.16}} & \ensuremath{-528.06^{\scriptscriptstyle -495.67}_{\scriptscriptstyle -530.21}} & \ensuremath{-598.72^{\scriptscriptstyle -561.71}_{\scriptscriptstyle -601.60}} \\
\textbf{LR+D r=64} & \ensuremath{-114.71^{\scriptscriptstyle -110.13}_{\scriptscriptstyle -115.78}} & \ensuremath{-238.07^{\scriptscriptstyle -226.17}_{\scriptscriptstyle -240.90}} & \ensuremath{-366.83^{\scriptscriptstyle -345.24}_{\scriptscriptstyle -368.67}} & \ensuremath{-500.74^{\scriptscriptstyle -469.98}_{\scriptscriptstyle -502.61}} & \ensuremath{-571.62^{\scriptscriptstyle -536.24}_{\scriptscriptstyle -573.30}} \\
\textbf{LR+D r=128} & \ensuremath{-101.11^{\scriptscriptstyle -98.28}_{\scriptscriptstyle -101.71}} & \ensuremath{-213.51^{\scriptscriptstyle -203.73}_{\scriptscriptstyle -215.51}} & \ensuremath{-338.59^{\scriptscriptstyle -319.59}_{\scriptscriptstyle -339.90}} & \ensuremath{-472.71^{\scriptscriptstyle -444.28}_{\scriptscriptstyle -473.27}} & \ensuremath{-543.15^{\scriptscriptstyle -510.28}_{\scriptscriptstyle -544.60}} \\
\textbf{LR+D r=256} &  & \ensuremath{-196.79^{\scriptscriptstyle -190.03}_{\scriptscriptstyle -198.86}} & \ensuremath{-315.23^{\scriptscriptstyle -299.58}_{\scriptscriptstyle -315.42}} & \ensuremath{-448.12^{\scriptscriptstyle -422.73}_{\scriptscriptstyle -448.83}} & \ensuremath{-518.27^{\scriptscriptstyle -489.55}_{\scriptscriptstyle -519.89}} \\
\textbf{LR+D r=512} &  &  & \ensuremath{-301.76^{\scriptscriptstyle -289.46}_{\scriptscriptstyle -303.07}} & \ensuremath{-431.16^{\scriptscriptstyle -410.47}_{\scriptscriptstyle -433.43}} & \ensuremath{-505.01^{\scriptscriptstyle -477.51}_{\scriptscriptstyle -507.63}} \\
\textbf{LR+D r=1024} &  &  &  & \ensuremath{-425.11^{\scriptscriptstyle -406.19}_{\scriptscriptstyle -426.03}} & \ensuremath{-497.28^{\scriptscriptstyle -474.39}_{\scriptscriptstyle -499.02}} \\
\textbf{LR+D r=N} & \ensuremath{-98.24^{\scriptscriptstyle -95.97}_{\scriptscriptstyle -98.54}} & \ensuremath{-192.29^{\scriptscriptstyle -186.56}_{\scriptscriptstyle -193.32}} & \ensuremath{-299.63^{\scriptscriptstyle -287.62}_{\scriptscriptstyle -301.14}} &  &  \\
\midrule
\textbf{KS-DPP d=4} & \ensuremath{-101.12^{\scriptscriptstyle -96.81}_{\scriptscriptstyle -101.57}} & \ensuremath{-197.80^{\scriptscriptstyle -190.82}_{\scriptscriptstyle -205.27}} & \ensuremath{-318.22^{\scriptscriptstyle -298.38}_{\scriptscriptstyle -327.08}} & \ensuremath{-445.94^{\scriptscriptstyle -443.08}_{\scriptscriptstyle -465.16}} & \ensuremath{-517.33^{\scriptscriptstyle -488.02}_{\scriptscriptstyle -544.71}} \\
\textbf{KS-DPP d=8} & \ensuremath{-99.18^{\scriptscriptstyle -96.88}_{\scriptscriptstyle -101.11}} & \ensuremath{-194.46^{\scriptscriptstyle -187.16}_{\scriptscriptstyle -194.92}} & \ensuremath{-300.89^{\scriptscriptstyle -287.87}_{\scriptscriptstyle -301.84}} & \ensuremath{-425.11^{\scriptscriptstyle -404.74}_{\scriptscriptstyle -425.25}} & \ensuremath{-498.78^{\scriptscriptstyle -469.65}_{\scriptscriptstyle -499.13}} \\
\textbf{KS-DPP d=16} & \ensuremath{-99.25^{\scriptscriptstyle -96.16}_{\scriptscriptstyle -99.56}} & \ensuremath{-193.19^{\scriptscriptstyle -186.49}_{\scriptscriptstyle -194.74}} & \ensuremath{-301.04^{\scriptscriptstyle -286.45}_{\scriptscriptstyle -301.95}} & \ensuremath{-423.70^{\scriptscriptstyle -403.78}_{\scriptscriptstyle -424.51}} & \ensuremath{-493.74^{\scriptscriptstyle -468.87}_{\scriptscriptstyle -495.69}} \\
\textbf{KS-DPP d=32} & \ensuremath{-99.12^{\scriptscriptstyle -96.17}_{\scriptscriptstyle -99.40}} & \ensuremath{-193.79^{\scriptscriptstyle -186.30}_{\scriptscriptstyle -194.33}} & \ensuremath{-299.17^{\scriptscriptstyle -286.61}_{\scriptscriptstyle -300.81}} & \ensuremath{-423.00^{\scriptscriptstyle -402.65}_{\scriptscriptstyle -424.42}} & \ensuremath{-493.22^{\scriptscriptstyle -468.52}_{\scriptscriptstyle -494.52}} \\
\bottomrule
\end{tabular}
\end{threeparttable}
\end{table}

\endgroup

\subsection{Full Result of Comparison with Exact MM}
Table \ref{tab:comparison_mm_full} provides the comprehensive results for the comparison between LR+D and KS-DPP against the exact MM method on the \ttM dataset.
This table extends the principal findings summarized in Section \ref{sec:ex} (Table \ref{tab:comparison_mm}) by detailing the test mean log-likelihoods across the full sweep of hyperparameter combinations for both $r$ and $d$.
Consistent with the discussion in the main text, KS-DPP outperforms the full-rank LR+D ($r=N$) at larger problem sizes, highlighting its superior suitability for scalable stochastic optimization.

\begin{figure*}
    \centering
    \begin{minipage}{0.51\textwidth}
        \includegraphics[width=\columnwidth]{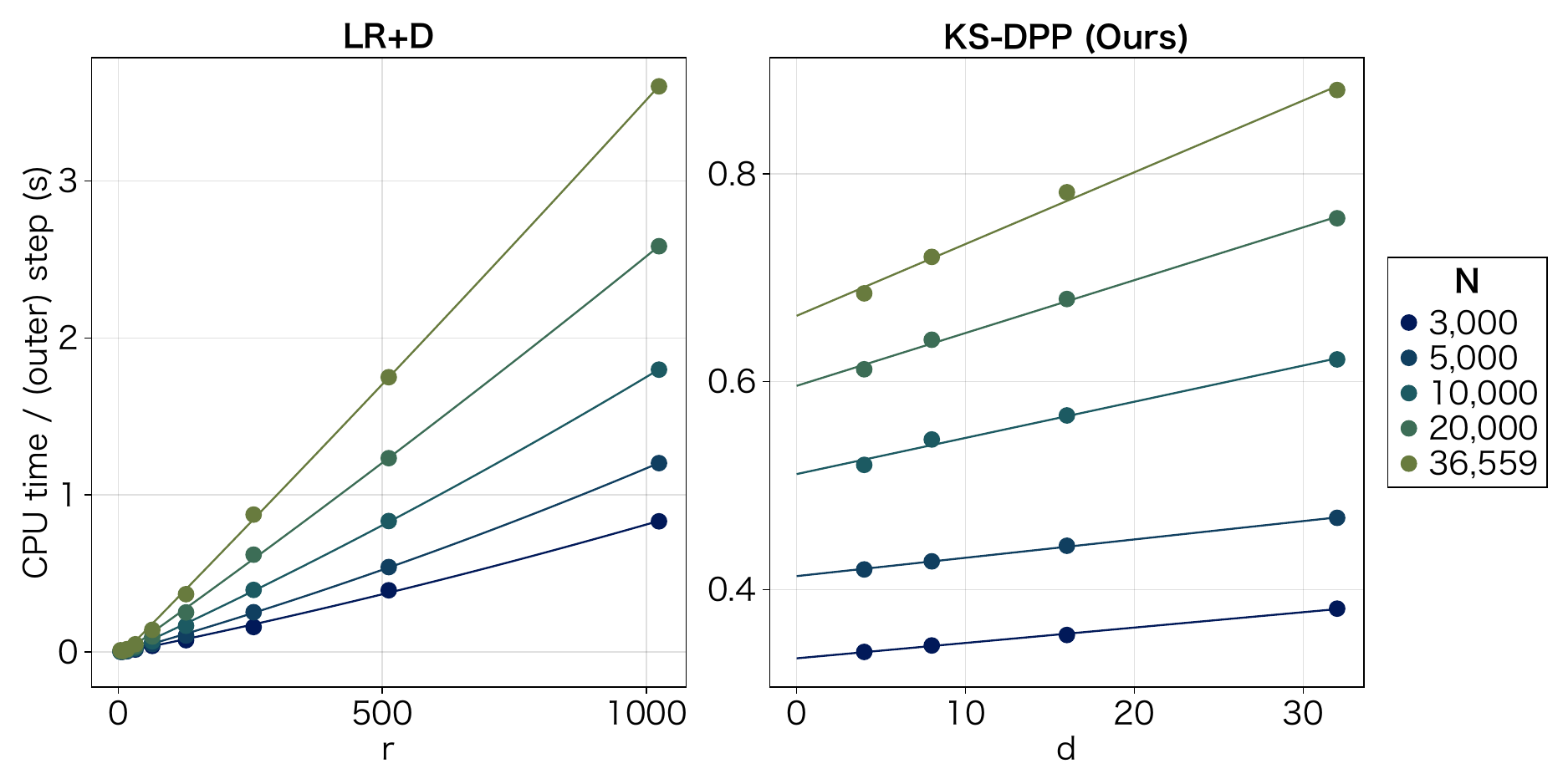}
        \caption{$r$ or $d$ vs. per-(outer-)step computational costs on \ttI. Each marker shows the mean of $3$ trials and the lines are quadratic (left) or linear (right) fits.} 
        \label{fig:cost_r_d_instacart}
    \end{minipage}
    \hspace{0.03\textwidth}
    \begin{minipage}{0.43\textwidth}
        \includegraphics[width=\columnwidth]{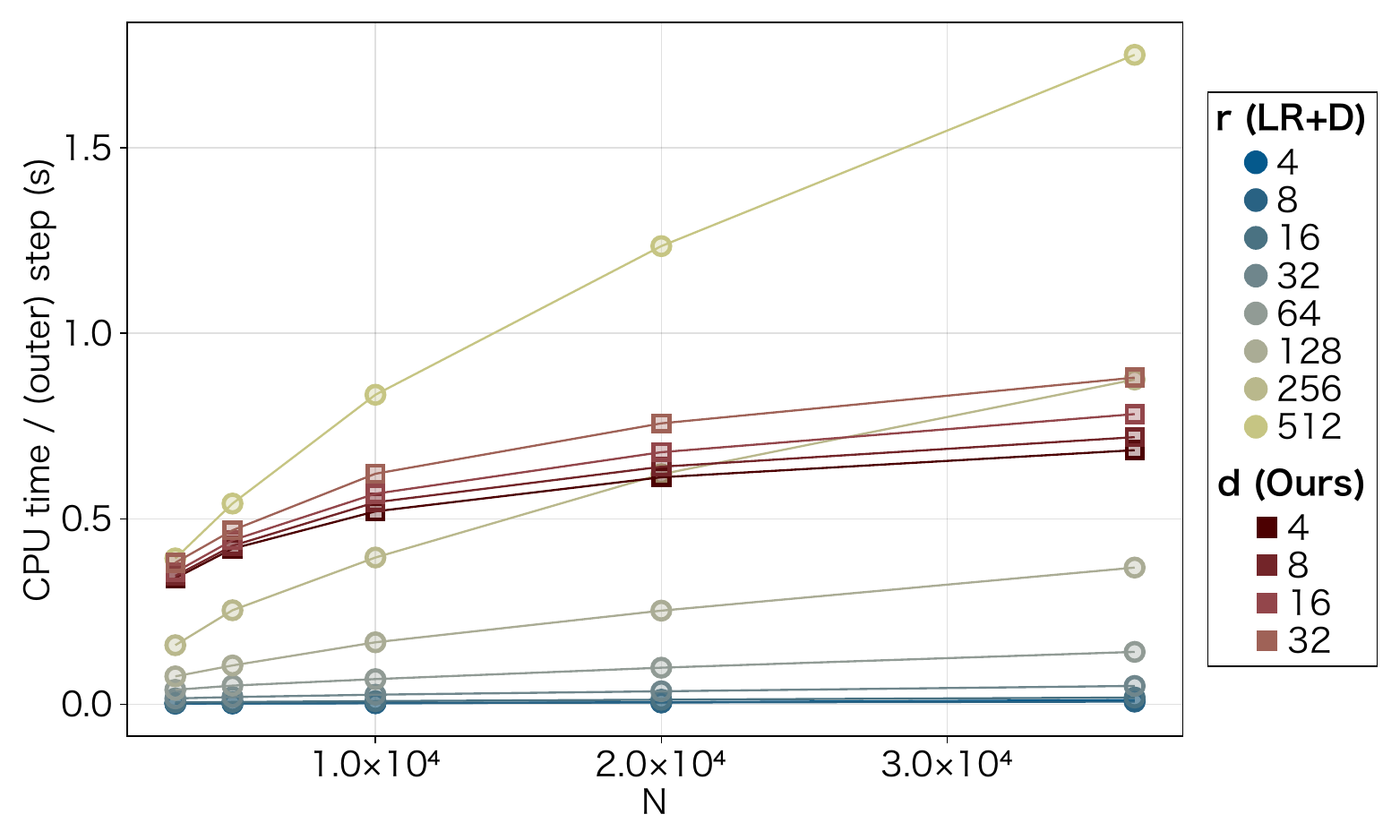}
        \caption{$N$ vs. per-(outer-)step computational costs on \ttI. Each marker shows the mean of $3$ trials.}
        \label{fig:cost_N_instacart}
    \end{minipage}
\end{figure*}

\subsection{Measured Computational Cost for Instacart}\label{app:measurement}
We conducted the same measurements as in Section \ref{sec:measurement} using the \ttI dataset.
The fitting functions used for $r$ and $d$ are identical to those in the main text (i.e., quadratic for LR+D and linear for KS-DPP).
The results are summarized in Figure \ref{fig:cost_r_d_instacart} and Figure \ref{fig:cost_N_instacart}.

As shown in Figure \ref{fig:cost_r_d_instacart} (right), the linearity of the outer step cost with respect to $d$ remains clearly visible for KS-DPP.
On the other hand, for LR+D (left), the CPU time exhibited a nearly linear behavior with respect to $r$ within the evaluated range, despite the quadratic fitting model.
This trend is likely because the experiments were conducted on a workstation equipped with 36 threads.
The optimized parallel matrix calculations in the system probably reduced the effect of the theoretical $O(r^2)$ complexity, so the quadratic curve did not show up clearly.

Regarding the dependence on $N$, similar to the observations on the \ttM dataset, a clear linear increase in cost is not prominently observed within the evaluated range of $N$ in Figure \ref{fig:cost_N_instacart}.
Instead, the computational cost exhibits a flatter and milder growth as $N$ increases, rather than escalating strictly linearly.

\end{document}